\documentclass{article}
\usepackage{graphicx} % Required for inserting images
\usepackage{amsmath}
\usepackage{amsfonts}
\usepackage{natbib}
\usepackage{bbding}
\usepackage{multirow}
\usepackage{subcaption}
\usepackage{hyperref} 
\usepackage{enumerate}
\usepackage[capitalise]{cleveref}
\usepackage{mathtools}
\usepackage{comment}
\usepackage{color}
\usepackage{url}
\usepackage[disable]{todonotes}
\usepackage[letterpaper, 
            left=1in, 
            right=1in, 
            top=1in, 
            bottom=1in]{geometry}
\title{\large{\textsc{Free-RBF-KAN: Kolmogorov–Arnold Networks with Adaptive Radial Basis Functions for Efficient Function Learning}}}

\author{\normalsize\scshape Shao-Ting Chiu\thanks{Department of Electrical and Computer Engineering, Texas A\&M University, College Station, TX, USA.
Email: \href{mailto:stchiu@tamu.edu}{stchiu@tamu.edu},
\href{mailto:ulisses@tamu.edu}{ulisses@tamu.edu}.
}\hspace{-1em}  
\and 
\normalsize\scshape Siu Wun Cheung\thanks{Center for Applied Scientific Computing, Lawerance Livermore National Laboratory, Livermore, CA, USA. 
This work was performed under the auspices of the U.S. Department of Energy by Lawerance Livermore National Laboratory under Contract DE-AC52-07NA27344 (LLNL-JRNL-2012535) and was supported by the LLNL-LDRD program under Project No. 25-ERD-051 and the LLNL Computing Scholar Program.
Email: \href{mailto:cheung26@llnl.gov}{cheung26@llnl.gov},
\href{mailto:li50@llnl.gov}{li50@llnl.gov},
\href{mailto:lee1029@llnl.gov}{lee1029@llnl.gov}.
}\hspace{-1em} 
\and 
\normalsize\scshape Ulisses Braga-Neto\footnotemark[1]\hspace{-1em} 
\and 
\normalsize\scshape Chak Shing Lee\footnotemark[2]\hspace{-1em}
\and 
\normalsize\scshape Rui Peng Li\footnotemark[2]
}

\date{}

\newcommand\todoin[2][]{\todo[inline, caption={2do}, #1]{
\begin{minipage}{\textwidth-4pt}#2\end{minipage}}}

\def\v{\textbf}
\def\m{\textbf}

\usepackage{amsthm}
\newtheorem{theorem}{Theorem}[section]
\newtheorem{lemma}[theorem]{Lemma}
\newtheorem{remark}[theorem]{Remark}
\newtheorem{corollary}[theorem]{Corollary}

\begin{document}

\maketitle

\begin{abstract}
    Kolmogorov–Arnold Networks (KANs) offer a promising framework for approximating complex nonlinear functions, yet the original B-spline formulation suffers from significant computational overhead due to De Boor's algorithm. While recent RBF-based variants improve efficiency, they often sacrifice the approximation accuracy inherent in the original spline-based design. To bridge this gap, we propose Free-RBF-KAN, an architecture that integrates adaptive learning grids and trainable smoothness parameters to enable expressive, high-resolution function approximation. Our method utilizes learnable RBF shapes that dynamically align with activation patterns, and we provide the first formal universal approximation proof for the RBF-KAN family. Empirical evaluations across multiscale regression, physics-informed PDEs, and operator learning demonstrate that Free-RBF-KAN can achieve accuracy comparable to its B-spline counterparts while delivering significantly faster training and inference. These results establish Free-RBF-KAN as an efficient and adaptive alternative for high-dimensional structured modeling tasks.
\end{abstract}

\section{Introduction}

The Kolmogorov–Arnold Network (KAN) \citep{liu2024kan} is a neural architecture grounded in the Kolmogorov–Arnold representation theorem \citep{kolmogorov2009representation}, which states that any multivariate continuous function can be expressed as a superposition of univariate continuous functions and addition. The original KAN leverages B-splines to model these univariate components due to their strong approximation capabilities and lack of spectral bias \citep{wang2024expressiveness}. However, computing B-spline bases using De Boor’s iteration and the necessary domain rescaling during training introduces significant computational overhead. To mitigate this, recent works, e.g., FastKAN \citep{li2024fastkan} and FasterKAN \citep{Athanasios2024} have explored alternative basis functions, such as Gaussian radial basis functions (RBFs) and reflectional switch activation functions (RSWAF), to reduce costs while retaining accuracy. Further improvements in stability and accuracy have been achieved through adaptive meshing in FreeKnots-KAN \citep{zheng2025free}, which learns the optimal knot placement.
These modifications can be viewed as extensions of RBF networks \citep{orr1996introduction}, which replace traditional sigmoidal nonlinearities with RBFs. While RBF networks are universal function approximators, they typically struggle with the curse of dimensionality, as the number of grid centers 
required to
achieve uniform approximation accuracy 
grows exponentially with input dimension. In contrast, RBF-KAN leverages univariate RBF basis functions within the superpositional structure of the Kolmogorov–Arnold theorem, enabling scalable function approximation in high-dimensional settings.

In this paper, we propose Free-RBF-KAN, which adopts a hierarchical, multichannel structure integrating adaptive learning grids (free knots) with tunable RBF shape parameters. This architecture constrains the grid to a fixed domain while learning a mesh aligned with activation patterns, balancing computational efficiency and expressiveness. Each RBF kernel acts as a univariate component whose superposition constructs the multivariate function. By permitting dynamic repositioning of grid points and smoothness during training, Free-RBF-KAN decouples the mesh from fixed uniform structures. Furthermore, our approach aligns with findings that spline-based KANs can serve as preconditioning for multichannel 
multilayer perceptrons (MLPs), leading to improved optimization landscapes and faster convergence \citep{actor2025leveraging}.

As a theoretical contribution, we formally extend the universal approximation theorem for RBF networks to the RBF-KAN architecture. To our knowledge, this is the first universality proof for the RBF-KAN family. Unlike B-spline KANs, RBF-KAN exhibits an approximation error bound independent of the target function and does not require a predesigned decomposition. This inherent universality demonstrates that RBF-KAN is a fundamentally powerful framework rather than just a strategy for reducing overhead. We also analyze the Neural Tangent Kernel (NTK) of Free-RBF-KAN, confirming that it does not exhibit spectral bias.

Finally, we evaluate Free-RBF-KAN on high-dimensional regression, physics-informed neural networks (PINNs), and operator learning. Using a physics-informed loss \citep{raissi2019a}, we solve heat conduction and  Helmholtz PDEs. Notably, the 
PINN baseline fails to converge on the heat equation and produces larger errors on the Helmholtz benchmark, despite using two orders of magnitude more parameters. We also evaluate Free-RBF-KAN as the trunk network within a DeepONet \citep{lu2019deeponet} for learning solution operators of reaction–diffusion PDEs. Free-RBF-KAN achieves lower approximation error than standard DeepONets and DeepOKANs while requiring fewer parameters. Across all experiments, Free-RBF-KAN consistently outperforms both RBF-KAN and the original KAN.

\subsection{Main Contributions}
The central contribution of this work is the development of Free-RBF-KAN, a novel and highly efficient KAN architecture utilizing a flexible RBF formulation.  Our work provides the following specific advancements:

\begin{itemize}
    \item \textbf{Architectural Innovation:}  Free-RBF-KAN is based on a free RBF formulation that leverages adaptive meshing (centroids) and tunable sharpness factors. This innovation grants the activation functions enhanced flexibility, enabling a dynamic alignment of the mesh representation with activation patterns to improve accuracy without increasing computational complexity.
    \item \textbf{Theoretical Foundation:} We formally extend the RBF network universality approximation theorem to the RBF-KAN family of neural networks. Furthermore, an NTK analysis confirms that Free-RBF-KAN exhibits the desirable property of lacking spectral bias in regression, akin to the original KAN.    
    \item \textbf{Broad Application and Efficiency:} We demonstrate the scalability of Free-RBF-KAN across diverse regimes, including general regression problems, physics-informed machine learning, and operator learning. Physics-informed Free-RBF-KAN and Free-RBF-KAN-ONet achieves comparable or superior accuracy to the original PINN and DeepONet, while using a smaller  number of parameters and being clearly superior to  RBF-KAN and KAN variants.
\end{itemize}

\section{Related Work}

A wide range of studies have explored both theoretical developments and applications of KAN and RBF frameworks. Physics-informed RBF networks using Gaussian kernels have demonstrated superior performance on high-frequency PDEs \citep{bai2023physics}, while \citet{zeng2024rbf} proposed RBF-PINN as an alternative to Fourier embeddings.
\citet{ss2024chebyshev} introduced Chebyshev basis functions within KAN, while \citet{wang2024kolmogorov} employed cubic B-splines to encode physical laws in strong, energy, and inverse forms. Hybrid models such as BSRBF-KAN \citep{ta2024bsrbfkan} combine B-splines and RBFs to harness the benefits of both smooth local representation and adaptive flexibility. \citet{krisnawan2025rbf} combines KAN and RBF networks for accurate indoor localization using RSSI-based fingerprinting. \citet{bai2023physics} demonstrated that physics-informed RBF networks can outperform traditional PINNs, using a single RBF layer. Similarly, \citet{ta2024bsrbfkan} combined B-spline and RBF bases to enhance training, while \citet{shukla2024comprehensive} and \citet{abueidda2025deepokan} introduced DeepOKAN, 
%\todo{[Solved] Ulisses: "DeepOKAN" was introduced by the Crunch group in "A comprehensive and FAIR comparison between MLP and KAN representations for differential equations and operator networks" \\ Steven: added} 
which employs RBF-KAN for operator learning. \citet{li2024fastkan} demonstrates that B-Spline can be approximated by RBFs with Gaussian kernels. A thorough comparison of PINN and KAN-based models is provided by \citet{shukla2024comprehensive}. 

On the theoretical side, classical results on RBF networks include their universal approximation capabilities \citep{park1993approximation, ismayilova2024universal}, convergence and consistency properties \citep{xu1994radial}, and optimal approximation results \citep{girosi1990networks}. RBF networks have also been used for solving multiscale PDEs \citep{wang2023solving}. Earlier work by \citet{wettschereck1991improving} focused on learning RBF centers to improve performance. Connections between RBFs and kernel machines were explored by \citet{que2016back}, while \citet{chen2024gaussian} extended KAN to a Gaussian process formulation that provides error estimates. Theoretical refinements of Kolmogorov’s representation theorem include the improvement by \citet{fridman1967improvement}, a constructive proof by \citet{braun2009constructive}, and an extension by \citet{kuurkova1991kolmogorov} showing that the representation can be realized through affine and sigmoidal functions--permitting the approximation of discontinuous but bounded functions. This theory underpins recent developments such as the KKAN model proposed in \citet{toscano2025kkans}, which builds on the Kurkova–Kolmogorov–Arnold representation. A broader perspective on KAN developments is provided by \citet{somvanshi2024survey}, highlighting applications across scientific computing, time-series forecasting, and graph learning.

\section{Fundamentals}

\subsection{Kolmogorov-Arnold Theorem}

KAN \citep{liu2024kan} is inspired by the Kolmogorov-Arnold representation theorem \citep{kolmogorov2009representation}. This theorem states that 
every multivariate continuous function 
$f\colon [0,1]^{n}\to \mathbb {R}$ can be represented as a superposition of continuous univariate functions.
%{\bf Theorem 1} (\citep{kolmogorov2009representation})  
%for any integer $n\geq 2$ there exist continuous univariate real functions $\phi^{(p,q)}(x): [0,1]\to \mathbb{R}$ and $\phi^{(q)}:\mathbb{R}\to\mathbb{R}$ that can represent continuous multivariate real function $f(x_1, \dots, x_d): [0,1]^d \to \mathbb{R}$ as 
%\begin{equation}
%    f(x_1, \dots, x_d) = \sum_{q=1}^{2d+1} \phi^{(q)} \left(\sum_{p=1}^{d} \phi^{(p,q)} 
%(x_p)\right)\,.
%\label{eq-KAN-theorem}    
%\end{equation}
%This result provides a universal representation for high-dimensional functions using only sums and compositions of univariate functions. 
However, the superposition may involve non-smooth inner functions and is hard for exact representation \cite{girosi1989representation}.

\subsection{Radial Basis Functions (RBFs)}
RBFs are real-valued functions whose output depends solely on the distance from a central point. 
RBF networks are widely used in interpolation, approximation theory, and machine learning. Classical results \citep{park1991universal}
show that
RBFs are universal approximators of continuous functions on compact domains. 
In the one-dimensional setting, 
a single-layer RBF network takes the form 
\begin{equation}
    g(x) \,=\, \sum_{m=1}^{G} \omega_m K\left(\frac{x-c_m}{\sigma_m}\right),
    \label{eq:rbf}
\end{equation}
where $G\in \mathbb{N^+}$ is the number of nodes 
% $x \in \mathbb{R}$ is the input, 
and kernel $K:\mathbb{R} \rightarrow \mathbb{R}^+$ depends only on $|x-c_m|$.
Each term of \eqref{eq:rbf} 
is parameterized by a weight
$\omega_m$, a centroid $c_m$, and a smoothness factor $\sigma_m > 0$.
In this work, we focus on one-dimensional RBFs as the fundamental 
building blocks for 
approximating multivariate functions through 
the superpositional structure provided by
the Kolmogorov–Arnold Theorem.

\section{Free RBF-KAN}
Assuming that the target function 
is sufficiently smooth, KAN employs B-spline basis functions for functions $\phi^{(p,q)}$ in \cref{lemma:KART}. 
Its expressivity derives from a multilayer architecture, an extension not present in the original formulation of \citet{kolmogorov2009representation} but one that substantially improves practical performance.
RBF-KAN preserves the  structure of \citet{liu2024kan}, while replacing the B-splines with one-dimensional RBFs. Let $n_l$ denote the number of nodes in the $l$-th layer and $\v{x}^{(l)}$ the inputs to the layer. A multilayer RBF-KAN  satisfies the following recursive relation:
\begin{equation}
\v{x}^{(l+1)} \,=\, 
    \Phi^{(l)}\left(\v{x}^{(l)}\right) \,=\, \begin{bmatrix}
    \displaystyle
    \sum_{j=1}^{n_l}
    \displaystyle
    \sum_{m=1}^{G} \omega^{(l)}_{1jm} \ 
    K\left( \frac{x_{j}^{(l)} - c^{(l)}_{1jm}}
                 {\sigma^{(l)}_{1jm}}\right) \\
    \vdots\\
    \displaystyle
    \sum_{j=1}^{n_l}
    \displaystyle
    \sum_{m=1}^{G} \omega^{(l)}_{n_{l+1},jm} \ 
    K\left( \frac{x_{j}^{(l)} - c^{(l)}_{n_{l+1},jm}}
                 {\sigma^{(l)}_{n_{l+1},jm}}\right) 
    \end{bmatrix}\,,
    \label{eq:rbfkan}
\end{equation}
{where $K:\mathbb{R}\to\mathbb{R}^+$ is an RBF kernel that is
assumed to be uniformly continuous.} As defined previously, with 
centroid $c_m \in \mathbb{R}$ and smoothness parameter $\sigma_m >0$, 
a common choice for $K$ is the Gaussian kernel: 
\begin{equation}
\label{eq:gaussian}
K\left(\frac{x-c_m}{\sigma_m}\right) \,=\, \exp\left(\frac{(x-c)^2}{\sigma}\right),
\end{equation}
for which $K \in C^{\infty}(\mathbb{R})$.
Another widely used option is the Mat\'ern kernel with 
smoothness parameter, such as $\nu=5/2$
%\begin{equation}
%K\left(\frac{x-c_m}{\sigma_m}\right) \,=\, \left(1+\frac{|x-c|}{\sigma} + \frac{5(x-c)^2}{3\sigma^2}\right)\exp\left(-\frac{\sqrt{5}|x-c|}{\sigma}\right),
%\end{equation}
which satisfies $K \in C^{3}(\mathbb{R})$.  
The smoothness of the chosen kernel  directly 
determines the smoothness of the resulting RBF-KAN.
%on local support, and adaptive resolution depends on the target solution. 
In this work, we investigate trainable centroids and smoothness 
parameters for additional flexibility and better 
representation quality.
Throughout the paper, {RBF-KAN} refers to networks 
with fixed centroids and smoothness parameters,
%\emph{Free-G-RBF-KAN} denotes architectures with trainable centroids but fixed smoothness factors; 
whereas {Free-RBF-KAN} denotes models where 
both centroids and smoothness factors are learnable.

In physics-informed machine learning, residual connections through nonlinear activations and scaling have been shown to improve accuracy \citep{liu2024kan}. 
Following the approach of \citet{wang2024kolmogorov}, 
we introduce an optional scaling matrix and a nonlinear activation
into each layer. The output of layer $l$ is then given by 
\begin{equation}
    \v{x}^{(l+1)} \,=\, \rho_o\left(\m{W}_{rbf} \odot \Phi^{(l)}(\v{x}^{(l)})  + \m{W} \odot \rho(\v{x}^{(l)})\right)
\end{equation}
where $\rho_o$ and $\rho: \mathbb{R} \to \mathbb{R}$ are componentwise activation functions. For the hidden layers, $\rho$ is  SiLU activation, following \cite{liu2024kan}, and $\rho_o$ is the sigmoid nonlinearity, which generally performs better than tanh in this architecture, except for the final output layer, where $\rho_o$ is the identity function.
The parameters of the RBF-KAN network are summarized in Table~\ref{tbl:param}. 

\begin{table}[h]
    \centering 
    \tiny
    \begin{tabular}{ccl}
    Trainable Param. & \#Param & \multicolumn{1}{c}{Description}\\
    \hline
    $\omega^{(l)}_{ijm}$ & $G \times n_{l+1} \times n_l$ & Weight of a radial basis function\\
    $c^{(l)}_{ijm}$ & $G \times n_{l+1} \times n_l$ & Centroid of a radial basis function\\
    $\sigma^{(l)}_{ijm}$ & $G \times n_{l+1} \times n_l$ & Smoothness of a radial basis function\\
    $W_{rbf}$ & $n_{l+1}\times n_l$ & Scaling factors of activation function $\phi$\\ 
    $W$ & $n_{l+1} \times n_l$ & Scaling factors of activation function $\sigma$\\
    \hline
    \end{tabular}
    \caption{Trainable parameters in RBF-KAN.}
    \label{tbl:param}
\end{table}
\vspace{-0.5cm}
\subsection{Adaptive Meshing}

Free-knots methods for KAN have been explored in \citet{zheng2025free} and \citet{actor2025leveraging}. 
In this work, we extend the underlying idea to the RBF-KAN setting. 
Unlike B-splines which require maintaining a strictly ordered sequence
of grid points, 
RBFs introduce additional flexibility of allowing any order of centroids, as
each kernel is evaluated independently and does not
rely on a prescribed ordering of centroids.
This allows us to develop a moving grid method in the RBF framework that can adaptively remesh the activation functions during training. 
It is worth noting that the free-knot adaptation for B-splines is 
considerably more complicated and 
the smoothness of the B-splines, determined by their polynomial orders, is fixed and
cannot be treated as a trainable parameter. 
In contrast, RBFs enable both the centroid locations and the associated 
smoothness parameters to be learned directly and efficiently, yielding a
more expressive and computationally efficient approach.

To obtain an efficient representation, we constrain the grids to 
lie within a prescribed domain during training. 
This can be achieved by reparameterization using 
a bounded monotonic activation function $\rho: \mathbb{R} \to [a,b]$ with $a<b$ and $\rho \in C^{\infty}$. Given a grid domain $\Omega_g \in (x_l,x_r)$ 
with $x_l < x_r$, a free parameter $\tilde{c} \in \mathbb{R}$ is mapped to 
a valid centroid location 
$c \in (x_l, x_r)$ via % 
%\begin{equation}
 $c = x_l + \frac{x_r - x_l}{b-a}\left(\rho(\tilde{c}) - a\right)$.
%    \label{eq:loc}
%\end{equation}
Unless otherwise specified, we set $a=-1, b=1$, and choose $\rho$ to be the tanh function,
ensuring that the centroid $c$ remains within $\Omega_g$.
This reparameterization is smooth and compatible with 
gradient-based optimization. For initialization, the grid points are 
placed uniformly within the domain. 

\subsection{Adaptive Smoothness}
The smoothness of RBFs can be specified
by the order parameter $\nu$ in the case of the Mat\`{e}rn kernel, 
or more generally
by the smoothness factor $\sigma$. 
%In Free-RBF-KAN, this smoothness factor is treated as a trainable parameter shared across all RBFs. 
%\todo{[FIXED] True? All shared? STC: for each $\sigma$. as in \cref{tbl:param}}
To ensure that $\sigma$ remains positive during 
gradient-based optimization, 
we introduce an unconstrained parameter $\tilde{\sigma} \in \mathbb{R}$ 
and define $\sigma$ via the mapping
$\sigma = \exp(\tilde{\sigma})$.
%
%This setting works for Adam optimization and all training work in this study. 
%In some cases, to further enhance the stability from preventing $\tilde{\sigma}$ to create small $\sigma$ that induce instability. \cref{eq:loc} can be applied to hard constrain the lower bound of $\sigma$. 
%
The combination of adaptive meshing and adaptive smoothness 
can significantly enhance the express power of Free-RBF-KAN. 
As demonstrated in \citet{zheng2025free}, 
allowing the RBF centroids to move within an
extended range 
improves gradient smoothness and helps training stability.
Although the introduction of additional trainable 
parameters  moderately increases the training cost, 
Free-RBF-KAN remains substantially faster than B-spline KAN. 
Moreover, the inference cost of Free-RBF-KAN is 
identical to
that of standard RBF-KAN once the centroids and smoothness parameters have 
been trained and fixed.
 
\subsection{Universal Approximation \label{sec:uniform}}
To establish the universal approximation
property 
of the Free-RBF-KAN architecture, 
we rely on several classical  
results from approximation theory.
We begin by recalling the Kolmogorov–Arnold representation theorem, which provides a canonical decomposition of 
multivariate continuous functions into sums of compositions of univariate functions.
\vspace{-\topsep}
\begin{lemma}[Kolmogorov-Arnold Representation Theorem \citep{kolmogorov2009representation, arnol1957functions}]
\label{lemma:KART}
For any continuous multi-variable function $f: [0,1]^d \to \mathbb{R}$, there exist $2d+1$ continuous univariate functions $\Phi^{(q)}: \mathbb{R} \to \mathbb{R}$ and $d(2d+1)$ continuous univariate functions $\phi^{(pq)}: [0,1] \to \mathbb{R}$ such that:
\begin{equation}
f(x_1, \dots, x_d) = \sum_{q=1}^{2d+1} \Phi^{(q)} \left( \sum_{p=1}^{d} \phi^{(pq)}(x_p) \right).
\label{eq:KART}
\end{equation}
\end{lemma}
We next recall a classical density result for ridge-function approximation.
\begin{theorem}[Pinkus Theorem \citep{Pinkus_1999}]
\label{theo:pinkus}
Let $\sigma \in C(\mathbb{R})$. The set
\begin{equation}
    \mathcal{M}(\sigma)=\mathrm{span}\{\sigma(\mathbf{w}\cdot \mathbf{x}-\theta)\; : \; \theta \in \mathbb{R}, \mathbf{w} \in \mathbb{R}^n\}
\end{equation}
is dense in $C(\mathbb{R}^n)$ on compact sets
with respect to uniform convergence, if and only if $\sigma$ is not a 
polynomial.
\end{theorem}
Restricting Theorem~\ref{theo:pinkus} to the one-dimensional setting yields a
useful univariate density result.
%, which will serve as a building block for our
%KAN approximation theorem.
%
\begin{lemma}[Univariate Density \citep{leshno1993multilayer}]
\label{lemma:Leshno}
Let $K: \mathbb{R} \to \mathbb{R}$ be a continuous function. The set of functions spanned by the shifts and scales of $K$, specifically 
\begin{equation}
\label{eq:S}
\mathcal{S} = \mathrm{span} \left\{ K\left(\frac{x-c}{\sigma}\right) \; : \; c, \sigma \in \mathbb{R}, \ \sigma \neq 0 \right\},
\end{equation}
is dense in the space of continuous functions $C[a, b]$ for any compact interval $[a, b] \subset \mathbb{R}$ if and only if $K$ is not a polynomial.
\end{lemma}

\begin{remark}
Lemma~\ref{lemma:Leshno} follows from Theorem~\ref{theo:pinkus}
by restricting to the one-dimensional case $n=1$ and
reparameterizing the affine arguments as shifts and scales.
Lemma~\ref{lemma:Leshno} provides a univariate density result that will be used
to approximate the univariate component functions appearing in the
Kolmogorov--Arnold and KAN-type representations.
\end{remark}

We are now ready to state the universal approximation theorem for 
\emph{non-polynomial KAN} (NP-KAN). We provide a sketch of the proof here, and a full proof in Appendix~\ref{sec:proof}.

\begin{theorem}[Universal Approximation of NP-KAN]
Let $K: \mathbb{R} \to \mathbb{R}$ be a continuous, non-polynomial function. 
For any $f \in C([0,1]^d)$ and any $\varepsilon > 0$, 
there exists an NP-KAN network $g$ of the form
\begin{equation}
    g(x_1, \dots, x_d) = \sum_{q=1}^{2d+1} \widehat{\Phi}^{(q)} \left( \sum_{p=1}^{d} \widehat{\phi}^{(pq)}(x_p) \right),
\end{equation}
where $\widehat{\Phi}^{(q)}, \widehat{\phi}^{(pq)} \in \mathcal{S}$ defined in \eqref{eq:S},
for all $1 \leq p \leq d$ and $1 \leq q \leq 2d+1$, 
such that 
\begin{equation} \label{eq:f-g}
\| f - g \|_{C([0,1]^d)} < \varepsilon \ .
\end{equation}
\label{thm:np-kan}
\end{theorem}
\vspace{-0.9cm}
\begin{proof}
(Sketch, see \cref{sec:proof} for the full proof.) By Lemma \ref{lemma:KART}, any $f\in C([0,1]^d)$ admits decomposition
\eqref{eq:KART} with continuous 
univariate functions
$\Phi^{(q)}$ and $\phi^{(pq)}$. The inner sum $S_q(x)=\sum_p\phi^{(pq)}(x_p)$
has compact range on which $\Phi^{(q)}$ is uniformly continuous.
Since $K$ is continuous and non-polynomial, Lemma \ref{lemma:Leshno} implies that $\mathcal{S}$ is dense in $C[a,b]$ for any compact interval. So, $\forall \varepsilon>0$,  each
$\phi^{(pq)}$
 and  
$\Phi^{(q)}$
 can be uniformly approximated on their respective 
 compact domains by  
 $\widehat\phi^{(pq)}, \widehat\Phi^{(q)}\in\mathcal S$ with errors chosen small enough so that uniform continuity controls the composition error.
Summing the resulting error bounds over $q$ yields \eqref{eq:f-g}.
\end{proof}

\begin{corollary}[Universal Approximation of RBF-KAN]
\label{cor:rbf-kan}
Let $K(x)$ be the Gaussian RBF in~\eqref{eq:gaussian},
which is continuous and non-polynomial.
For any continuous function $f \in C([0,1]^d)$ and any $\varepsilon>0$,
there exists an RBF-KAN network $g$ such that
$
\|f-g\|_{C([0,1]^d)} < \varepsilon.
$
\end{corollary}

\section{Numerical experiments}

We now turn to numerical experiments to assess the practical performance of the proposed Free-RBF-KAN architecture. 
To ensure a fair and meaningful comparison, we begin by outlining the configurations of the baseline models. 
The KAN used in our comparisons employs cubic B-spline basis functions (i.e., B-splines of order 3). 
All models in the KAN family, including KAN and the various RBF-KAN counterparts, are constructed with the same number of nodes and layers, 
providing a controlled setting for evaluating expressive power and computational efficiency. 
Unless otherwise noted, all RBF-KAN variants, including Free-RBF-KAN, use the Gaussian kernel as their radial basis activation.

\subsection{Functional approximation }
\label{sec:funcapprox}
%\todo[info,inline]{\url{https://github.com/stevengogogo/sparse-RBF-KAN/blob/main/exps/regression2D.py}}
Our first experiment examines the approximation of a nonsmooth function to illustrate the importance of adaptive meshing. Following \cite{actor2025leveraging}, we consider the function defined for $(x,y) \in (0,1)^2$:
\begin{equation} \notag
f(x,y) = \cos(4\pi x) + \sin(\pi y) + \sin(2\pi y) + |\sin(3\pi y^2)|
\end{equation}
We evaluate a compact neural network architecture with layer sizes [2, 5, 1] across several model variants: MLP, KAN, FreeKnots-KAN, RBF-KAN, and Free-RBF-KAN. 
FreeKnots-KAN~\citep{actor2025leveraging} introduces adaptive remeshing 
of the B-spline grid, enabling improved approximation accuracy 
compared to fixed-grid KANs. 
All models are trained using the LBFGS optimizer for 300 epochs 
with a learning rate of 1 on a regression dataset of 16,384 points 
processed in batches of 1,024.

In \cref{tbl:nonsmooth}, the results show that the RBF-KAN substantially outperforms the standard KAN 
while using fewer parameters. 
Both FreeKnots-KAN and its RBF-based counterpart achieve similarly low test errors; 
however Free-RBF-KAN attains this accuracy with fewer parameters, 
underscoring its efficiency. 
These findings demonstrate that replacing B-splines with RBFs with Gaussian kernels 
not only improves accuracy but also reduces model complexity. 
Overall, RBF-KAN exhibits strong capability in approximating nonsmooth functions, with the Gaussian kernel providing the best performance among RBF variants considered.

\subsection{Spectral bias}

%\todo[info, inline]{Source code: \url{https://github.com/stevengogogo/sparse-RBF-KAN/blob/main/exps/Regression/regression_mlp_kan.py}}
Our next experiment investigates the spectral bias of Free-RBF-KAN.
Unlike MLPs, which are known to exhibit spectral bias, 
KAN has been shown to avoid this limitation \citep{wang2024kolmogorov}. 
We assess whether this property extends to RBF-KAN by conducting 
a spectral bias analysis using the Neural Tangent Kernel (NTK) framework, 
following the methodology in \cite{wang2024kolmogorov}. 
Detailed NTK derivation can be found in \citet[Sec. 4.1]{wang2024kolmogorov} and \citet[Sec. 3.1]{wang2021eigenvector}. 
The NTK analysis is performed on a multiscale regression task: 
\begin{equation}
    f(x) = 0.1\sin(50\pi x) + \sin(2\pi x), \quad x \in [0,1],
\label{eq:f}
\end{equation}
using MLP, KAN, and RBF-KAN variants. 
The input domain is discretized into 100 uniform grid points. 
The MLP architecture consists of 4 hidden layers with 100 neurons each, 
while all KAN variants use 3 hidden layers with 5 neurons. 
KAN employs cubic B-spline activations with 20 grid points, followed by 
a tanh normalization layer. 
RBF-KAN and Free-RBF-KAN use the same grid specifications 
but replace the B-splines with Gaussian RBFs. 
All models are trained using the mean square error (MSE) loss.
%\begin{equation}
%\mathcal{L}(\theta) = \frac{1}{n} \sum_{i=1}^n (f(x_i; \theta) - y_i)^2.
%\end{equation}
All KAN variants successfully learn the multiscale regression problem and outperform the MLP 
baseline, as shown in \cref{fig:func-approx} for the approximated solutions
and in \cref{fig:func-approx-training-loss} for the training loss. 
Moreover, the NTK eigenvalue spectra,
shown in \cref{fig:ntk},
reveal that RBF-KAN maintains a broader and  less rapidly decaying NTK eigenvalue
spectrum comparable to 
that of KAN and MLP, indicating an absence of spectral bias (\cref{fig:ntk-indx}). Free-RBF-KAN exhibits an even wider eigenvalue spectrum after 9000 training steps, suggesting that the additional flexibility provided by adaptive centroids and shape parameters promotes faster convergence.

Overall, the NTK analysis demonstrates that RBF-KAN not only trains more efficiently than the B-spline-based KAN but also retains the ability to represent multiscale features more effectively than MLPs. Furthermore, Free-RBF-KAN offers additional improvements in convergence behavior due to its adaptive smoothness and mesh refinement capabilities.

\begin{comment}
\begin{figure}[htbp]
\begin{subfigure}[t]{0.5\textwidth}
    \includegraphics[width=\linewidth]{prediction_reg.pdf}
     %\begin{subfigure}[t]{0.33\textwidth}
     %   \caption{Eigenfunction}
     %   \includegraphics[width=\linewidth]{eigenfunction_reg.pdf}
     %   \label{fig:eigfunc}
     %\end{subfigure}%
     \caption{The approximation of $f$ in \eqref{eq:f}
     using  MLP, KAN, RBF-KAN, and Free-RBF-KAN.
     \label{fig:func-approx}}
     \end{subfigure}
    \begin{subfigure}[t]{0.5\textwidth}
\centering\includegraphics[width=\textwidth]{loss_reg.pdf}
     \caption{The training loss of approximating $f$ in \cref{eq:f}
     using  MLP, KAN, RBF-KAN, and Free-RBF-KAN.
     \label{fig:func-approx-training-loss}}
\end{subfigure}

\begin{subfigure}[t]{0.5\textwidth}
\centering\includegraphics[width=\textwidth]{eigenvalue_reg.pdf}
        \label{fig:ntk-indx}
     %\begin{subfigure}[t]{0.33\textwidth}
     %   \includegraphics[width=\linewidth]{relative_eigenvalue_reg.pdf}
     %\end{subfigure}%
     \caption{The NTK analysis on the spectral bias of MLP, KAN, RBF-KAN, Free-RBF-KAN in approximating $f$ in \cref{eq:f}}
     \label{fig:ntk}
\end{subfigure}
\caption{NTK analysis of MLP and KAN variations on a multiscale function.}
\end{figure}
\end{comment}
%\begin{comment}
\begin{figure}[h]
        \includegraphics[width=\linewidth]{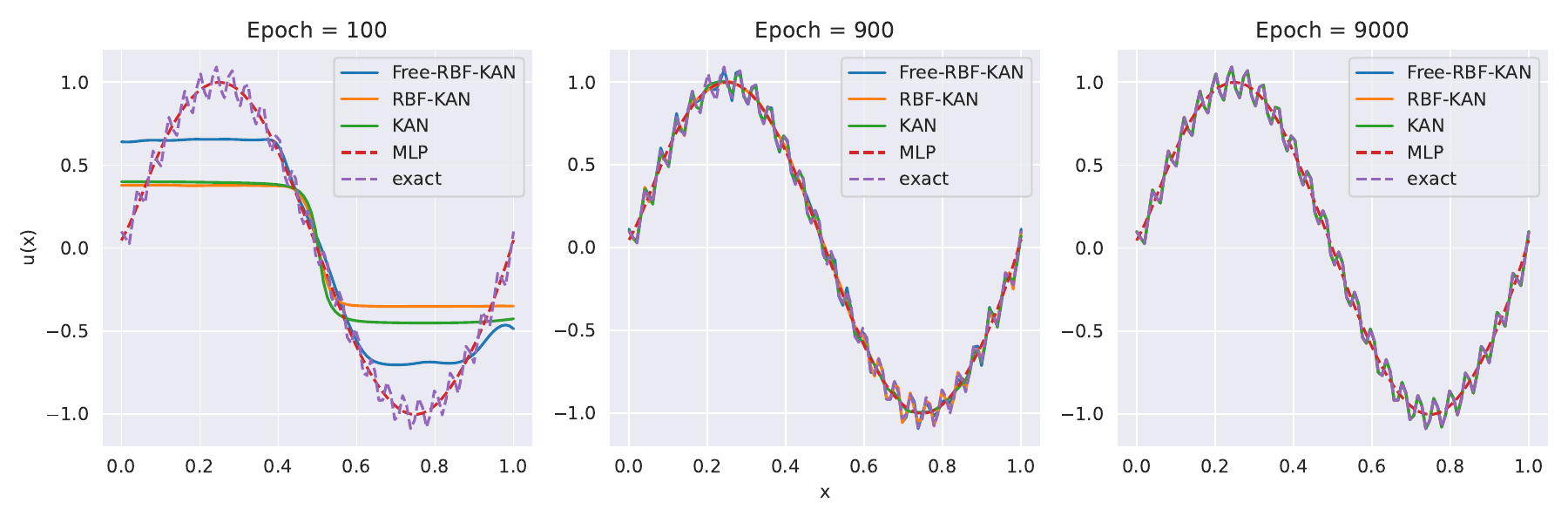}
     %\begin{subfigure}[t]{0.33\textwidth}
     %   \caption{Eigenfunction}
     %   \includegraphics[width=\linewidth]{eigenfunction_reg.pdf}
     %   \label{fig:eigfunc}
     %\end{subfigure}%
     \caption{The approximation of $f$ in \eqref{eq:f}
     using  MLP, KAN, RBF-KAN, and Free-RBF-KAN.
     \label{fig:func-approx}}
\end{figure}%
\begin{figure}[h]
\centering\includegraphics[width=0.45\textwidth]{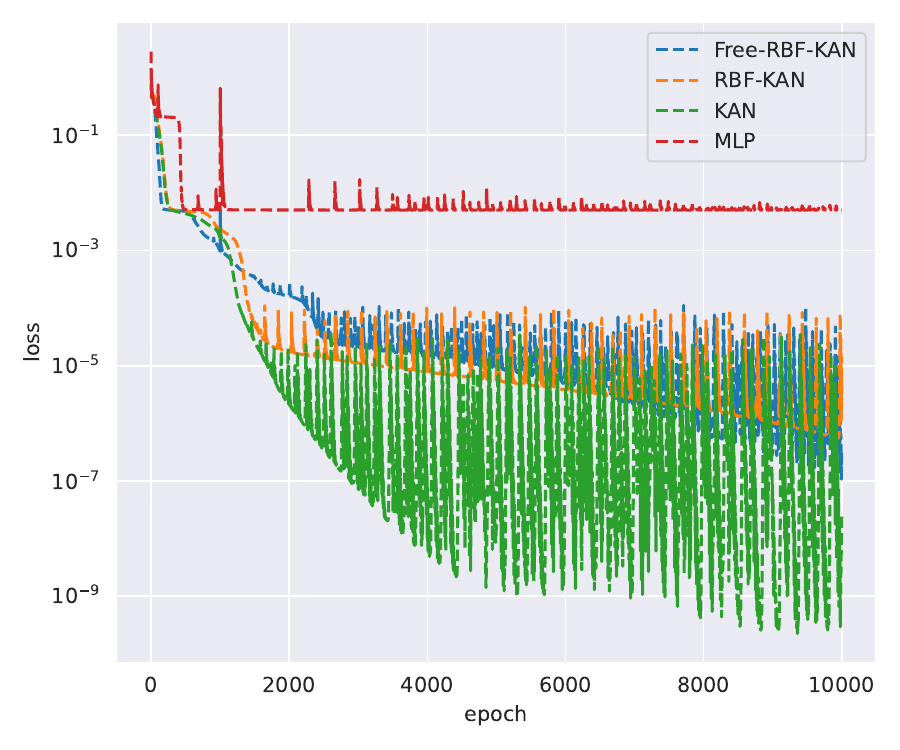}
     \caption{The training loss of approximating $f$ in \cref{eq:f}
     using  MLP, KAN, RBF-KAN, and Free-RBF-KAN.
     \label{fig:func-approx-training-loss}}
\end{figure}%
\begin{figure}[h]
\centering\includegraphics[width=0.45\textwidth]{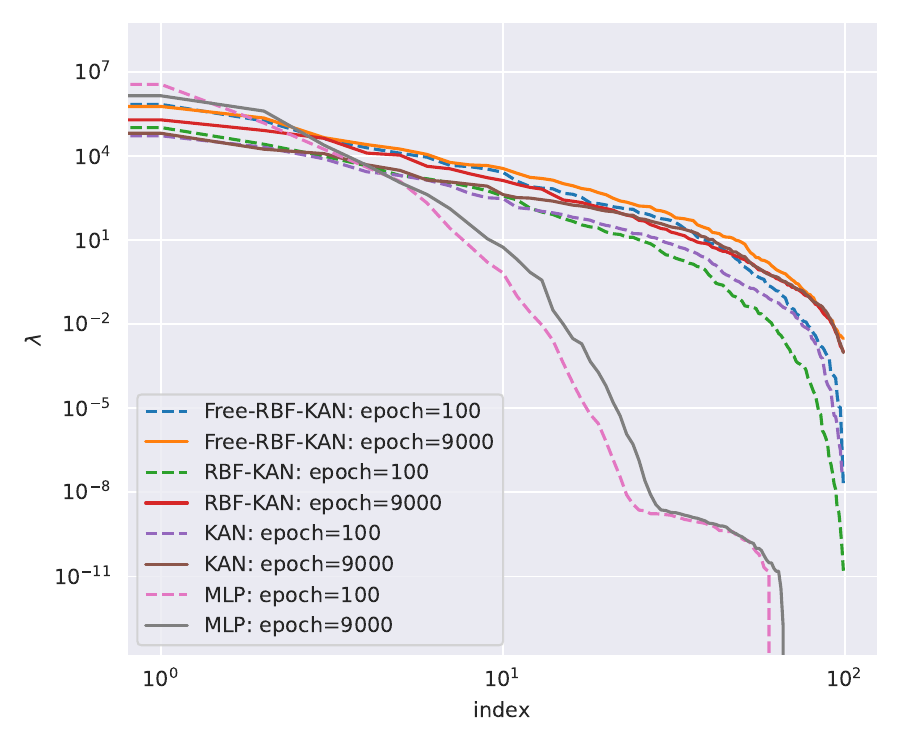}
        \label{fig:ntk-indx}
     %\begin{subfigure}[t]{0.33\textwidth}
     %   \includegraphics[width=\linewidth]{relative_eigenvalue_reg.pdf}
     %\end{subfigure}%
     \caption{The NTK analysis on the spectral bias of MLP, KAN, RBF-KAN, Free-RBF-KAN in approximating $f$ in \cref{eq:f}}
     \label{fig:ntk}
\end{figure}
%\end{comment}

\subsection{High-dimensional regression problem}
\label{sec:mnist}
 
%\todo{Ulisses: What is the task?}
%\todo[info,inline]{\url{https://github.com/stevengogogo/sparse-RBF-KAN/blob/main/exps/regression_mnist.py}}

\begin{table*}[tp]
    \centering 
    \scriptsize
    \begin{tabular}{lcccccc}
         \multicolumn{7}{c}{\textbf{Regression}}\\
        \hline
        Problem&Model & Layers & Basis & \#Param. & Test MSE & Training time (sec.)\\
        \hline
        \multirow{5}{5em}{\textbf{2D Nonsmooth Function}}
        &MLP & [2,10,10,10,1] & Tanh & 261 & 5.26e-1 & 162\\
        &KAN & [2,5,1] & B-Spline  & 195 & 3.54e-3 & 124\\
        &FreeKnots-KAN & [2,5,1] & B-Spline   & 307 & 3.54e-4 & 136\\
        &RBF-KAN & [2,5,1] & RBF-Gaussian & 120 & 5.67e-3 & 122\\
        &Free-RBF-KAN & [2,5,1] & RBF-Gaussian & 450 & 2.74e-4 & 121\\
        %KAN & [2,5,1] & RBF-RSWF & \checkmark & 290 & 1.77e-2\\
        %KAN & [2,5,1] & RBF-InvMQ & & 120 & 4.26e-2\\
        \hline
         &Model& Layers & Basis  & \# Param. & Test loss& Training time (sec.)\\
        \hline
        %\multicolumn{6}{c}{MNIST}\\
        %\hline
        \multirow{4}{5em}{\textbf{MNIST}}
        &MLP & $[28\times28,64,10]$ & Tanh&  509410 & 6.702e-02&81.58\\
        &KAN & $[28\times28,64,10]$ & B-Spline&  762240 & 1.166e-01&97.95\\
        &RBF-KAN & $[28\times28,64,10]$ &RBF-Gaussian&   508160 & 2.020e-01&82.27\\
        %RBF-KAN-G & [28*28,64,10] & RBF& \checkmark & & 516640 & 9.548e-02&84.99\\
        &Free-RBF-KAN & $[28\times28,64,10]$ &RBF-Gaussian& 525120 & 8.789e-02& 85.81\\[0.5ex]
        \hline
        \multicolumn{7}{c}{\textbf{Physics-Informed Machine Learning}}\\
        \hline 
        &Model & Layers & Basis & \# Param. & $L^{2}$-loss &Training time (sec.)\\
        \hline
        \multirow{4}{5em}{\textbf{2D Heat Conduction}}
        &MLP& [2,40,40,40,1] & Tanh & 5081 & 1 & 60\\
        &KAN & [2,5,5,1] &B-Spline&  1400 & 6.52e-3& 267\\
        &RBF-KAN&[2,5,5,1]&RBF-Gaussian&1280&2.78e-3&124\\
        &Free-RBF-KAN&[2,5,5,1]&RBF-Gaussian&2000&2.41e-3& 138\\
        %KAN & [1,5,1] & Spline & \checkmark & 292 & 4.17e-4 &\\
        \hline
        \multirow{4}{5em}{\textbf{2D Helmholtz}}
        &MLP& [2,128,128,128,1] & Tanh  & 50049 & 4.15e-2 & 39\\
        &KAN & [2,5,5,1] &B-Spline&  600 & 1.58 & 153\\
        &RBF-KAN&[2,5,5,1]&RBF-Gaussian&400&3.67e-1& 49\\
        &Free-RBF-KAN&[2,5,5,1]&RBF-Gaussian&640&3.35e-2& 62\\
        %KAN & [1,5,1] & Spline & \checkmark & 292 & 4.17e-4 &\\
        \hline
        \multicolumn{7}{c}{\textbf{Operator Learning}}\\
        \hline
        &Model & Layers (DeepONet Trunk)  & Basis & \# Param & $L^{2}$-loss & Training time (sec.)\\ 
        \hline 
        \multirow{4}{5em}{\textbf{Reaction Diffusion PDE}}
        &MLP & [2,[40]*4,100] & Tanh & 18921 & 2.08e-2 & 78\\
        &KAN & [2, 4, 4, 4, 100] & B-Spline& 11945 & 6.15e-2 & 96\\
        &RBF-KAN & [2, 4, 4, 4, 100] &RBF-Gaussian& 10625 & 3.7e-2 & 84\\
        &Free-RBF-KAN & [2, 4, 4, 4, 100] &RBF-Gaussian&  11185 & 1.94e-2& 88\\ 
    \end{tabular}
    \vspace{1em}
    \caption{Performance benchmarks across regression, physics-informed, and operator learning tasks. We evaluate FreeKnots-KAN \citep{actor2025leveraging} on 2D nonsmooth function. All B-spline basis functions utilize a third-order smoothness. For operator learning tasks, the DeepONet branch architecture consists of an MLP with [100, [40]*4,100] neurons, evaluated across various trunk configurations.}
    \label{tbl:mnist}
    \label{tbl:heat}
    \label{tbl:helmholtz}
    \label{tbl:nonsmooth}
    \label{tb:deeponet}
\end{table*}

In the next set of experiments, we evaluate model performance on 
a high-dimensional regression task using the MNIST dataset\footnote{\url{https://docs.pytorch.org/vision/main/generated/torchvision.datasets.MNIST.html}}.
This setting is nontrivial for traditional RBF networks, as placing and optimizing
RBF centroids directly in high-dimensional spaces 
notoriously difficult and computationally expensive. 
On the other hand, RBF-KAN leverages a hierarchical architecture composed of 
univariate RBF functions, making it naturally suited to 
deep structures and far more effective at mitigating the curse of dimensionality. 
\todo{Ulisses: I believe that we should study the performance of RBF-KAN against that of KAN and perhaps other KAN variants, but not MLP. There are already many papers that compare KAN and MLP. Several KAN papers only compare KAN variants.}

For both RBF-KAN and KAN, we use 10 grid points for each activation function. 
The implementation of KAN follows the setup in \cite{liu2024kan}. 
To enhance performance, we remove the residual activation and apply a sigmoid function for normalization.
The MNIST data are normalized, and we use a batch size of 64. 
All networks share the same architecture with layer sizes 
$[28\times 28,64,10]$.
Training is performed for 20 epochs using the Adam optimizer with a learning rate of 1e-3.
Test loss is evaluated on a test dataset as shown in \cref{tbl:mnist}.
\todo{[Fixed] RL: what do you mean by "independent dataset" STC: Test dataset}

RBF-KAN exhibits  progressively improved performance when adaptive grids and
adaptive smoothness are incorporated. 
Prior work \citep{yu2024kan} has reported that KAN performs worse than MLP on the MNIST dataset. 
While RBF-KAN still lags behind MLP in accuracy and requires longer training time, 
Free-RBF-KAN achieves better accuracy than the standard KAN, 
and substantially narrows the performance gap to MLP. 
In terms of the training time, the additional flexibility of Free-RBF-KAN 
does not introduce noticeable overhead compared to  KAN, 
yet it consistently yields improved accuracy. 
These results demonstrate that the adaptivity in both the grid points and the kernel smoothness
can enhance the model performance without significantly compromising computational efficiency. 
Among all the KAN-based methods, RBF-KAN achieves the fastest 
training speed but also the lowest accuracy. As observed in prior 
work \cite{yu2024kan}, all KAN variants remain less competitive than MLP in terms of 
parameter efficiency, accuracy, and training time. 
Nonetheless, Free-RBF-KAN provides clear improvements in both accuracy and efficiency compared to RBF-KAN and standard KAN.

\begin{comment}
\begin{table*}[htp]
    \centering 
    \small
    \begin{tabular}{cccccc}
        \hline
        Type & Layers & Basis  & \# Param. & Test loss& Training time (sec.)\\
        %\hline
        %\multicolumn{6}{c}{MNIST}\\
        \hline
        MLP & $[28\times28,64,10]$ & Tanh&  509410 & 6.702e-02&81.58\\
        KAN & $[28\times28,64,10]$ & B-Spline&  762240 & 1.166e-01&97.95\\
        RBF-KAN & $[28\times28,64,10]$ & RBF&   508160 & 2.020e-01&82.27\\
        %RBF-KAN-G & [28*28,64,10] & RBF& \checkmark & & 516640 & 9.548e-02&84.99\\
        Free-RBF-KAN & $[28\times28,64,10]$ & RBF  & 525120 & 8.789e-02& 85.81\\
        \hline 
        
        \hline
    \end{tabular}
    \caption{Model performance on MNIST dataset. The Free-RBF-KAN, which is RBF-KAN with trainable centroids and smoothness 
    parameters of RBFs, yields the best loss.}
    \label{tbl:mnist}
\end{table*}
\end{comment}

\subsection{Physics-Informed Machine Learning} \label{sec:piml}

We evaluate RBF-KAN within the PINN framework~\citep{raissi2019a}. To approximate the solution $u$ of a PDE $\mathcal{N}[u]=f$ in domain $\Omega$ subject to boundary conditions $\mathcal{B}[u]=g$ on $\partial\Omega$, we minimize the composite residual loss %
\begin{comment}
\begin{equation}
    \begin{split}
    \mathcal{L}(\omega) &= \frac{1}{N_c}\sum_{i=1}^{N_c} \left| \mathcal{N}[\hat{u}(\mathbf{x}_i)] - f(\mathbf{x}_i) \right|^2\\
    &+ \frac{1}{N_b} \sum_{j=1}^{N_b} \left| \hat{u}(\mathbf{x}_j) - g(\mathbf{x}_j) \right|^2,
    \end{split}
    \label{eq:pinn-loss}
\end{equation}
\end{comment}
where derivatives are computed via automatic differentiation. We demonstrate this approach on two benchmarks: 2D heat conduction with high-frequency forcing and the 2D Helmholtz equation with smooth solution.
\subsubsection{Heat Conduction in 2D} \label{sec:heat}
\begin{comment}
\todo[info, inline]{
Source code: \url{https://github.com/stevengogogo/sparse-RBF-KAN/blob/main/exps/pinn_highfreq.py}
}

\todoin{
    [Idea] A good experiment we can do is screening k from 1 to 50, and screening on grids.  The 2D matrix has x-axis on k, y-axis on number of grids. My hypothesis is the Free-RBF-KAN can reach higher accuracy with few grids. similar to Fig 3.(f) in KINN paper\citep{wang2024kolmogorov}.
}
\end{comment}

The heat conduction problem from \citet{wang2021eigenvector} has been shown 
to be more effectively learned by
in KAN than by MLP \citep{wang2024kolmogorov}. 
To evaluate whether the proposed RBF-KAN 
can also handle multiscale physics-informed problems, 
we compare RBF-KAN and Free-RBF-KAN against MLP and the original KAN.
The governing equation is given by 
\begin{equation}
    \begin{aligned}
        u_t &= \frac{1}{(50\pi)^2} u_{xx}, & & \quad (x, t) \in (0,1) \times (0,1),\\
        u(x,0) &= \sin(50\pi x),  & & \quad x\in [0,1],\\
        u(x,t) &=0, & & \quad (x,t)\in\{0,1\} \times [0,1], 
    \end{aligned}
\end{equation}
and its analytical solution is given by %
%
%\begin{equation}
$u(x,t) = e^{-t} \sin(50 \pi x)$.
%\end{equation}
%
%In this experiment, we set $K=50$. 
We randomly sample 4000 interior collocation points and 200 boundary points 
for each boundary segment.  
The MLP baseline has 4 hidden layers with width 40, while all KAN variants 
use an architecture with layer sizes [2,5,5,1]  and 30 grid points per activation. 
The KAN model uses cubic B-spline activations, and the grid range of KAN variations is set as $(x_l,x_r)=(0,1)$, same as \citet{wang2024kolmogorov}. All models are trained using the 
Adam optimizer with a learning rate $10^{-3}$ and an exponential learning 
rate scheduler with $\gamma=0.999$ applied each epoch. 
Training is performed for 15,000 epochs.

The results presented in \cref{fig:heat}
show that both RBF-KAN and Free-RBF-KAN can 
successfully 
learn this high frequency problem. 
While RBF-KAN attains slightly lower accuracy than KAN, it trains in 
roughly half the time required by KAN. 
Incorporating adaptive grids and smoothness
further improves the performance:
Free-RBF-KAN achieves higher  accuracy than both KAN and RBF-KAN. 
Although the additional flexibilities in Free-RBF-KAN increase 
the training time slighly, Free-RBF-KAN remains substantially faster than KAN, 
as shown in \cref{tbl:heat}.

\begin{comment}
\begin{table*}[bh]
    \centering 
    \begin{tabular}{cccccc}
        \hline
        Model & Layers & Basis & \# Param. & $L^{\infty}$-loss &Training time\\
        \hline
        MLP& [2,40,40,40,1] & Tanh & 5081 & 1 & 60\\
        Free-RBF-KAN&[2,5,5,1]&RBF-Gaussian&2000&2.41e-3& 138\\
        RBF-KAN&[2,5,5,1]&RBF-Gaussian&1280&2.78e-3&124\\
        KAN & [2,5,5,1] & Spline&  1400 & 6.52e-3& 267\\
        %KAN & [1,5,1] & Spline & \checkmark & 292 & 4.17e-4 &\\
        \hline
    \end{tabular}
    \caption{Performance on 2D heat conduction problem. Grid size is set 30 for all KAN approaches. The training times are measured in seconds.}
    \label{tbl:heat}
\end{table*}
\end{comment}

\subsubsection{Helmholtz  Equation in 2D} \label{sec:helm}

%\todo[info, inline]{
%    Source code: \url{https://github.com/stevengogogo/sparse-RBF-KAN/blob/main/exps/pinn_helmotz.py}. More result on PR: (1) \url{https://github.com/stevengogogo/sparse-RBF-KAN/pull/21}
%}

To further demonstrate the benefits of adaptive meshing, we examine physics-informed learning for a 2D 
Helmholtz equation with Dirichlet boundary conditions, following the setup in 
\citet{wang2020understanding}. The problem is defined 
in the domain $\Omega = (-3,3)^2$ as
\begin{equation}
\begin{split}
    u_{xx} + u_{yy} + u &= q \quad \text{ in } \Omega,\\
    u&=0 \quad \text{ on } \partial\Omega,
\end{split}
\label{eq-helmotz-prob}
\end{equation}
where the forcing term is given to have %
%
%\begin{equation}
%\begin{split}
%q(x,y) = &-(a_1 \pi)^2 \sin(a_1 \pi x) \sin(a_2 \pi y)\\
%&- (a_2 \pi)^2 \sin(a_1 \pi x) \sin(a_2 \pi y) \\
%    &+ k^2 \sin(a_1 \pi x) \sin(a_2 \pi y),
%\end{split}
%\end{equation}
the analytical solution 
$u(x,y) = \sin(\pi x) \sin(\pi y)$. 
%We set $\Omega = (-3,3)^2$ and choose $a_1=a_2=k=1$. 
For each training epoch, we randomly sampled 4,000 collocation points in the interior and 
100 boundary points on each edge of the domain. 
The MLP baseline has 5 hidden layers with 128 nodes each, whereas 
all the KAN variants have 2 hidden layers of 5 nodes per layer and 10 grids 
for each activation function. 
Training is performed using the Adam optimizer with a learning rate of $10^{-3}$.

As shown in \cref{fig:helmholtz}, both the MLP and Free-RBF-KAN architectures successfully 
approximate the PDE solution. 
Although minor inaccuracies remain near the boundaries, 
these errors could be further reduced by explicitly enforcing boundary conditions within the network architecture \citep{wang2023expert}. 
In contrast, the standard RBF-KAN produces inferior approximations, 
highlighting the performance gains achieved through adaptive centroids and kernel shapes in the Free-RBF
approach. 
Notably, the original KAN fails to capture the solution over the entire domain, 
a limitation that is likely due to the high smoothness demands of the target function. 
While increasing the B-spline order could potentially improve KAN’s expressive capacity, doing so would incur substantial computational overhead.

Furthermore, the timing results show that KAN is significantly slower than the other approaches, indicating that the B-spline formulation is computationally expensive 
when combined with AD, as shown by the timings in \cref{tbl:helmholtz}. 
In contrast, RBF-KAN achieves a substantial speed-up. 
%The adaptive variant, namely Free-RBF-KAN, further improves the mean squared error (MSE), 
%outperforming both MLP and standard RBF-KAN. 
%These results demonstrate that RBF-KAN retains the advantages of both RBF networks and KAN-like 
%architectures while avoiding the curse of dimensionality. 
%Interestingly, KAN fails to learn this problem,  which may be attributed to insufficient smoothness in its activation functions, as suggested in \citet{samadi2024smooth}.

\begin{figure*}[htp]
    \centering
    \begin{subfigure}[t]{0.309\linewidth}
        \caption{Exact}
        \vspace{-0.23em}
        \includegraphics[width=\linewidth]{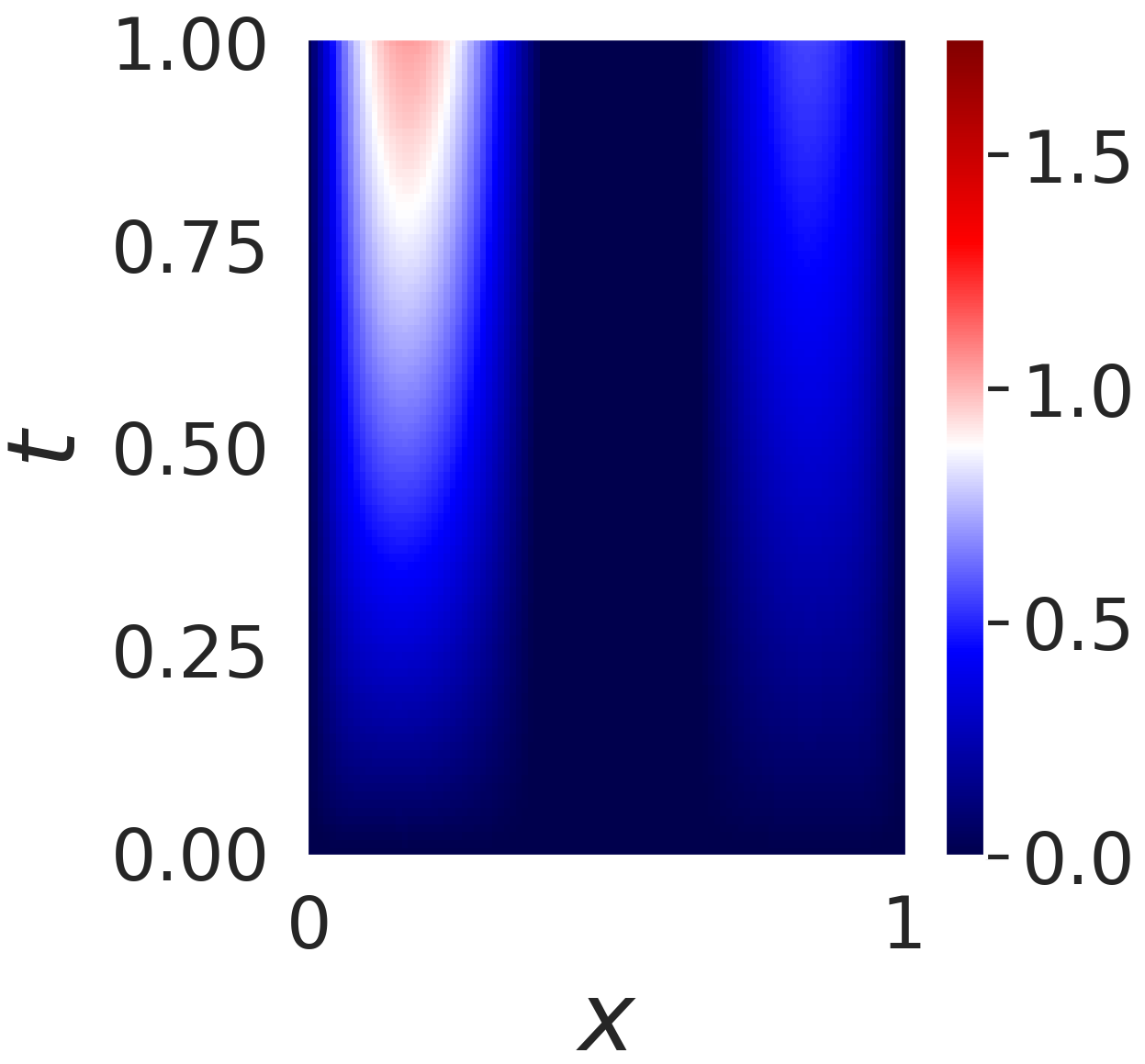}
    \end{subfigure}%
    \begin{subfigure}[t]{0.17\linewidth}
        \caption{KAN}
        \includegraphics[width=\linewidth]{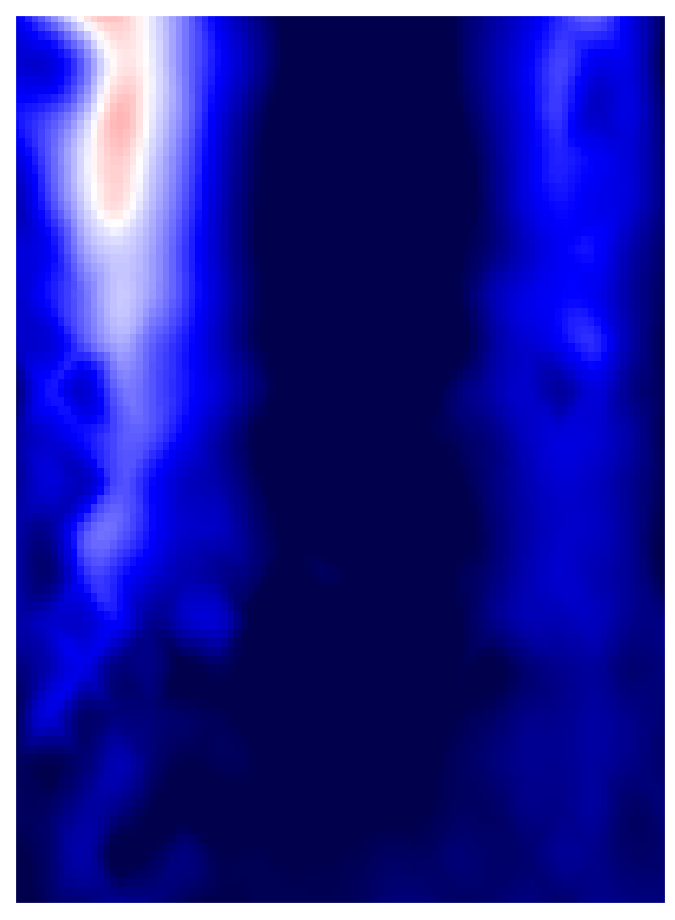}
    \end{subfigure}%
    \begin{subfigure}[t]{0.17\linewidth}
        \caption{MLP}
        \includegraphics[width=\linewidth]{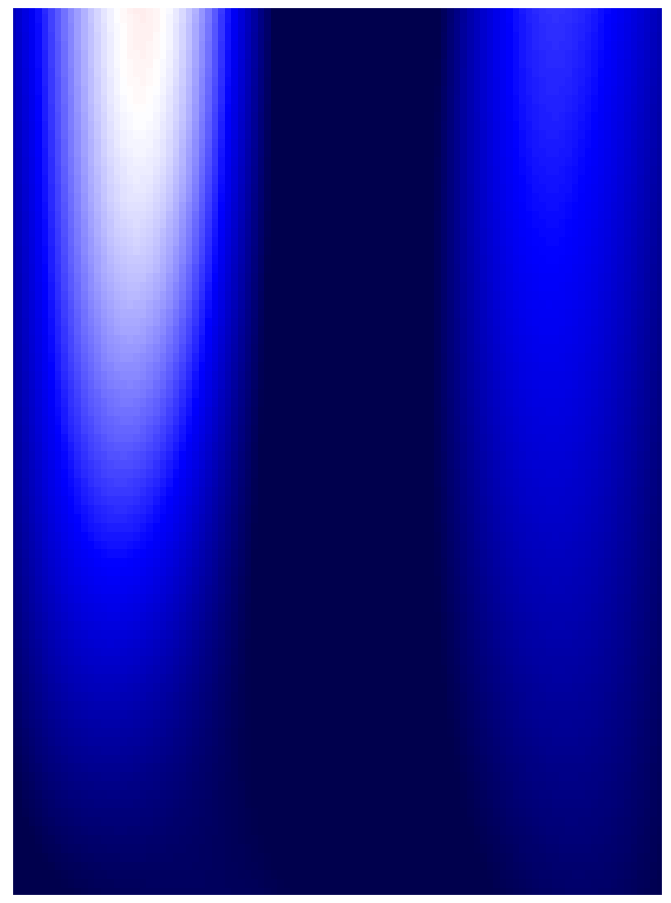}
    \end{subfigure}%
    \begin{subfigure}[t]{0.169\linewidth}
        \caption{RBF-KAN}
        \includegraphics[width=\linewidth]{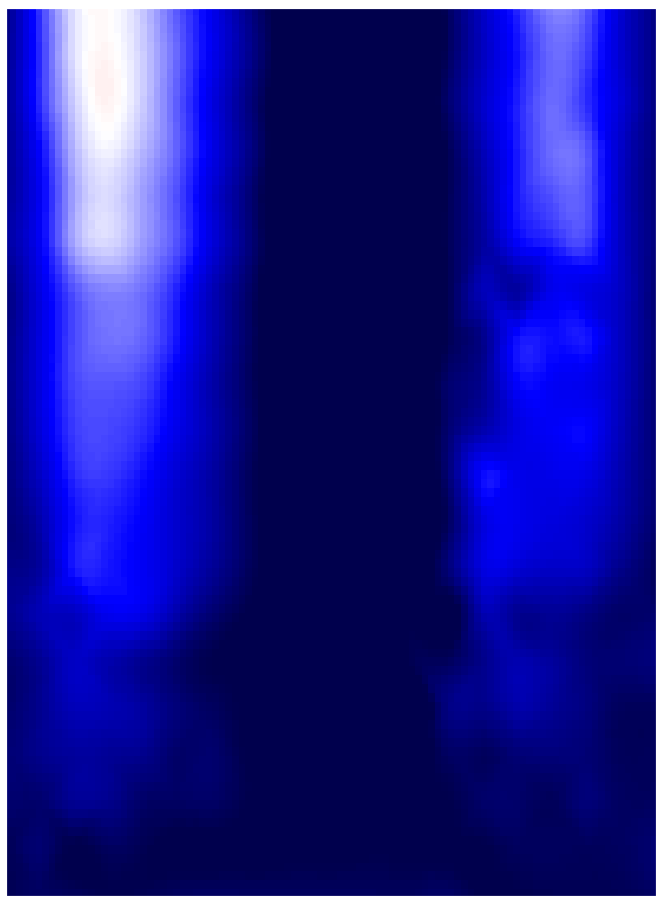}
    \end{subfigure}%
    \begin{subfigure}[t]{0.169\linewidth}
        \caption{Free-RBF-KAN}
        \includegraphics[width=\linewidth]{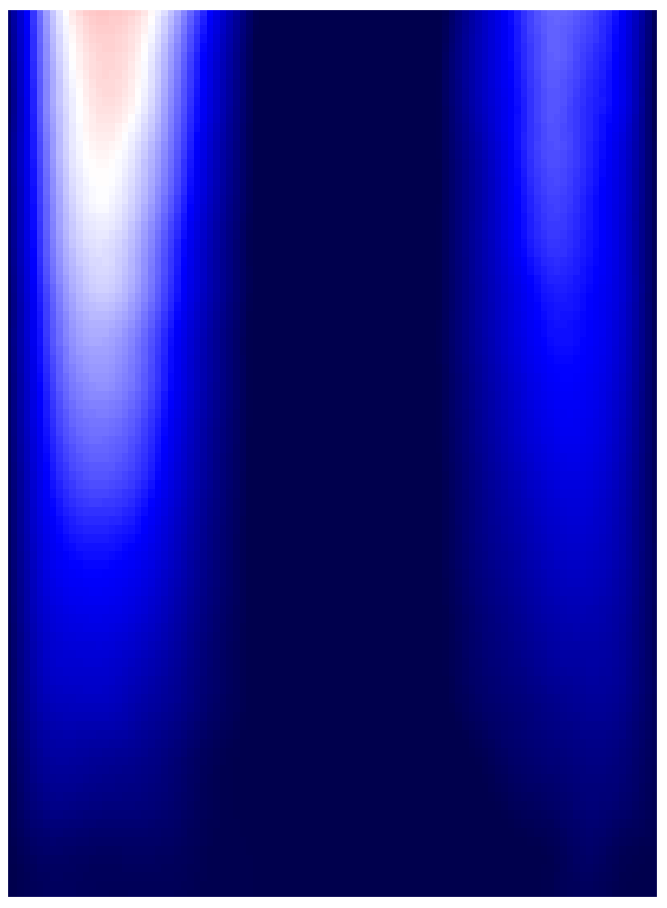}
    \end{subfigure}%
    \caption{The prediction of DeepONet variations on reaction-diffusion PDE. All models has same branch structure with MLP, but with different trunk architecture. A sampled force function ($f$) is used to generate the solution ($u(x,t)$). \label{fig:donet}}
\end{figure*}

\subsection{DeepONet with KAN Variants}

Our final set of experiments investigates the performance of different architectures within the DeepONet framework. DeepONet \citep{lu2019deeponet} consists of a branch network and a trunk network; in our study, we evaluate how various choices for the trunk network, 
including MLP, KAN, RBF-KAN, and Free-RBF-KAN, affect operator-learning accuracy, while the branch network is kept fixed as an MLP. The goal is to learn the
solution operator $\mathcal{G}: f(x) \mapsto u(x, t)$ for the 1D nonlinear reaction diffusion equation:
\begin{equation}
    u_t = \frac{1}{100} u_{xx} + \frac{1}{100}u^2 + f,\quad (x,t) \in (0,1) \times (0,1),
    \label{eq:rd}
\end{equation}
with zero boundary and initial condition. All DeepONet variations share the same branch network structure with 100 inputs, 100 outputs, 4 hidden layers of width 40. 
The trunk net is replaced by 
MLP, KAN, RBF-KAN, or Free-RBF-KAN, following the configurations summarized
in \cref{tb:deeponet}.

All models were trained with the Adam optimizer using 
a learning rate of $10^{-3}$, and an exponential learning rate scheduler 
($\gamma=0.95$) for 10,000 steps. 
Training data consist of 50 forcing functions sampled from a Gaussian Random Field (GRF) 
with the length scale $0.2$, where
in each sample 10 locations of sensor observation $u(x, t)$ are randomly selected  for training.
Performance was evaluated by the relative mean squared error 
(RMSE) on the full $100 \times 100$ spatiotemporal grid, 
and averaged over 30 random test seeds. 

\cref{fig:donet} visualizes the model predictions for a representative test function. 
Among the evaluated architectures, Free-RBF-KAN provides the most accurate approximation, 
slightly outperforming the MLP in capturing fine-scale features. 
In contrast, both the standard RBF-KAN and the original KAN 
exhibit inferior performance. 
These results underscore the importance of adaptive RBFs and 
suggest that the fixed B-spline bases used in KAN may 
lack sufficient smoothness to effectively represent the target solution. 

The MLP trunk employs fully connected layers of width 40, yielding the largest model size. In contrast, the KAN-based trunks use a much smaller hidden dimension of 4 and 20 grid points for their basis functions. KAN uses cubic B-splines, whereas the RBF-KAN variants employ Gaussian kernels. The results 
presented in \cref{tb:deeponet}
show that Free-RBF-KAN achieves the highest accuracy among all tested architectures while also requiring fewer parameters than the standard KAN.

\section{Conclusion}

In this paper, we introduced {Free-RBF-KAN}, an efficient evolution of Kolmogorov–Arnold networks that replaces fixed B-spline bases with adaptive Gaussian kernels. By jointly optimizing grid centroids and shape parameters, our method circumvents the recursive computational overhead of B-splines, achieving a significant reduction in training time on regression and PDE benchmarks  while simultaneously improving approximation accuracy. Theoretically, we extended the universal approximation theorem to the RBF-KAN family; empirically, we verified its lack of spectral bias and robust generalization in physics-informed regimes. Although dense MLPs remain the standard for unstructured high-dimensional data, our findings position {Free-RBF-KAN} as a compelling primitive for scientific computing where interpretability, sparsity, and derivative smoothness are paramount. Future directions include exploring dualities with Gaussian-activated MLPs to leverage mature optimization kernels for further scalability.

\section*{Impact Statement}

This research contributes to the fundamental advancement of Machine Learning. While such work inherently possesses broad societal implications, we do not identify any specific consequences that require distinct emphasis here.

%\section{Hierarchical training}
%HKAN \citep{dudek2025hkan} uses hieracical structure to train with least square
%For preventing overfitting, ridge least square can be used.
%\section{PINN for Hierarchical training}
%\section{Two step training: least square for height coefficient; GD for entire}

\bibliography{ref}

\begin{thebibliography}{40}
\providecommand{\natexlab}[1]{#1}
\providecommand{\url}[1]{\texttt{#1}}
\expandafter\ifx\csname urlstyle\endcsname\relax
  \providecommand{\doi}[1]{doi: #1}\else
  \providecommand{\doi}{doi: \begingroup \urlstyle{rm}\Url}\fi

\bibitem[Abueidda et~al.(2025)Abueidda, Pantidis, and
  Mobasher]{abueidda2025deepokan}
Diab~W Abueidda, Panos Pantidis, and Mostafa~E Mobasher.
\newblock {DeepOKAN}: Deep operator network based on {Kolmogorov} {Arnold}
  networks for mechanics problems.
\newblock \emph{Computer Methods in Applied Mechanics and Engineering},
  436:\penalty0 117699, 2025.

\bibitem[Actor et~al.(2025)Actor, Harper, Southworth, and
  Cyr]{actor2025leveraging}
Jonas~A. Actor, Graham Harper, Ben Southworth, and Eric~C. Cyr.
\newblock Leveraging {KANs} for expedient training of multichannel {MLPs} via
  preconditioning and geometric refinement.
\newblock \emph{arXiv preprint arXiv:2505.18131}, 2025.
\newblock URL \url{https://arxiv.org/abs/2505.18131}.

\bibitem[Arnol'd(1957)]{arnol1957functions}
Vladimir~Igorevich Arnol'd.
\newblock On functions of three variables.
\newblock In \emph{Doklady Akademii Nauk}, volume 114, pages 679--681. Russian
  Academy of Sciences, 1957.

\bibitem[Bai et~al.(2023)Bai, Liu, Gupta, Alzubaidi, Feng, and
  Gu]{bai2023physics}
Jinshuai Bai, Gui-Rong Liu, Ashish Gupta, Laith Alzubaidi, Xi-Qiao Feng, and
  YuanTong Gu.
\newblock Physics-informed radial basis network ({PIRBN}): A local
  approximating neural network for solving nonlinear partial differential
  equations.
\newblock \emph{Computer Methods in Applied Mechanics and Engineering},
  415:\penalty0 116290, 2023.

\bibitem[Braun and Griebel(2009)]{braun2009constructive}
J{\"u}rgen Braun and Michael Griebel.
\newblock On a constructive proof of {Kolmogorov's} superposition theorem.
\newblock \emph{Constructive approximation}, 30\penalty0 (3):\penalty0
  653--675, 2009.

\bibitem[Chen(2024)]{chen2024gaussian}
Andrew~Siyuan Chen.
\newblock Gaussian process {Kolmogorov-Arnold} networks.
\newblock \emph{arXiv preprint arXiv:2407.18397}, 2024.

\bibitem[Delis(2024)]{Athanasios2024}
Athanasios Delis.
\newblock Fasterkan.
\newblock \url{https://github.com/AthanasiosDelis/faster-kan/}, 2024.

\bibitem[Fridman(1967)]{fridman1967improvement}
Buma~L Fridman.
\newblock An improvement in the smoothness of the functions in {AN
  Kolmogorov}'s theorem on superpositions.
\newblock In \emph{Doklady Akademii Nauk}, volume 177, pages 1019--1022.
  Russian Academy of Sciences, 1967.

\bibitem[Girosi and Poggio(1989)]{girosi1989representation}
Federico Girosi and Tomaso Poggio.
\newblock Representation properties of networks: {Kolmogorov}'s theorem is
  irrelevant.
\newblock \emph{Neural Computation}, 1\penalty0 (4):\penalty0 465--469, 1989.

\bibitem[Girosi and Poggio(1990)]{girosi1990networks}
Federico Girosi and Tomaso Poggio.
\newblock Networks and the best approximation property.
\newblock \emph{Biological cybernetics}, 63\penalty0 (3):\penalty0 169--176,
  1990.

\bibitem[Ismayilova and Ismayilov(2024)]{ismayilova2024universal}
Aysu Ismayilova and Muhammad Ismayilov.
\newblock On the universal approximation property of radial basis function
  neural networks.
\newblock \emph{Annals of Mathematics and Artificial Intelligence}, 92\penalty0
  (3):\penalty0 691--701, 2024.

\bibitem[Kolmogorov(1956)]{kolmogorov2009representation}
A.N. Kolmogorov.
\newblock On the representation of continuous functions of several variables as
  superpositions of continuous functions of a smaller number of variables.
\newblock \emph{Dokl. Akad. Nauk}, 108\penalty0 (2), 1956.

\bibitem[Krisnawan et~al.(2025)Krisnawan, Mukti, and Purnomo]{krisnawan2025rbf}
Aditya~Bagus Krisnawan, Prasetiyono~Hari Mukti, and Mauridhi~Hery Purnomo.
\newblock {RBF-KAN}: Integrated approach for accurate indoor localization in
  dense grid {RSSI} fingerprint.
\newblock In \emph{2025 17th International Conference on Knowledge and Smart
  Technology (KST)}, pages 46--51. IEEE, 2025.

\bibitem[Kurkova(1991)]{kuurkova1991kolmogorov}
Vera Kurkova.
\newblock Kolmogorov's theorem is relevant.
\newblock \emph{Neural computation}, 3\penalty0 (4):\penalty0 617--622, 1991.

\bibitem[Leshno et~al.(1993)Leshno, Lin, Pinkus, and
  Schocken]{leshno1993multilayer}
Moshe Leshno, Vladimir~Ya Lin, Allan Pinkus, and Shimon Schocken.
\newblock Multilayer feedforward networks with a nonpolynomial activation
  function can approximate any function.
\newblock \emph{Neural networks}, 6\penalty0 (6):\penalty0 861--867, 1993.

\bibitem[Li(2024)]{li2024fastkan}
Ziyao Li.
\newblock {Kolmogorov-Arnold} networks are radial basis function networks.
\newblock \emph{arXiv preprint arXiv:2405.06721}, 2024.
\newblock URL \url{https://arxiv.org/abs/2405.06721}.

\bibitem[Liu et~al.(2024)Liu, Wang, Vaidya, Ruehle, Halverson,
  Solja{\v{c}}i{\'c}, Hou, and Tegmark]{liu2024kan}
Ziming Liu, Yixuan Wang, Sachin Vaidya, Fabian Ruehle, James Halverson, Marin
  Solja{\v{c}}i{\'c}, Thomas~Y Hou, and Max Tegmark.
\newblock {KAN}: {Kolmogorov-Arnold} networks.
\newblock \emph{arXiv preprint arXiv:2404.19756}, 2024.

\bibitem[Lu et~al.(2019)Lu, Jin, and Karniadakis]{lu2019deeponet}
Lu~Lu, Pengzhan Jin, and George~Em Karniadakis.
\newblock {DeepONet}: Learning nonlinear operators for identifying differential
  equations based on the universal approximation theorem of operators.
\newblock \emph{arXiv preprint arXiv:1910.03193}, 2019.

\bibitem[Orr et~al.(1996)]{orr1996introduction}
Mark~JL Orr et~al.
\newblock Introduction to radial basis function networks, 1996.

\bibitem[Park and Sandberg(1991)]{park1991universal}
Jooyoung Park and Irwin~W Sandberg.
\newblock Universal approximation using radial-basis-function networks.
\newblock \emph{Neural computation}, 3\penalty0 (2):\penalty0 246--257, 1991.

\bibitem[Park and Sandberg(1993)]{park1993approximation}
Jooyoung Park and Irwin~W Sandberg.
\newblock Approximation and radial-basis-function networks.
\newblock \emph{Neural computation}, 5\penalty0 (2):\penalty0 305--316, 1993.

\bibitem[Pinkus(1999)]{Pinkus_1999}
Allan Pinkus.
\newblock Approximation theory of the {MLP} model in neural networks.
\newblock \emph{Acta Numerica}, 8:\penalty0 143–195, 1999.
\newblock \doi{10.1017/S0962492900002919}.

\bibitem[Que and Belkin(2016)]{que2016back}
Qichao Que and Mikhail Belkin.
\newblock Back to the future: Radial basis function networks revisited.
\newblock In \emph{Artificial intelligence and statistics}, pages 1375--1383.
  PMLR, 2016.

\bibitem[Raissi et~al.(2019)Raissi, Perdikaris, and Karniadakis]{raissi2019a}
M.~Raissi, P.~Perdikaris, and G.~E. Karniadakis.
\newblock Physics-informed neural networks: {{A}} deep learning framework for
  solving forward and inverse problems involving nonlinear partial differential
  equations.
\newblock \emph{Journal of Computational Physics}, 378:\penalty0 686--707,
  2019.
\newblock ISSN 0021-9991.
\newblock \doi{10.1016/j.jcp.2018.10.045}.
\newblock URL
  \url{https://www.sciencedirect.com/science/article/pii/S0021999118307125}.

\bibitem[Shukla et~al.(2024)Shukla, Toscano, Wang, Zou, and
  Karniadakis]{shukla2024comprehensive}
Khemraj Shukla, Juan~Diego Toscano, Zhicheng Wang, Zongren Zou, and George~Em
  Karniadakis.
\newblock A comprehensive and fair comparison between {MLP} and {KAN}
  representations for differential equations and operator networks.
\newblock \emph{Computer Methods in Applied Mechanics and Engineering},
  431:\penalty0 117290, 2024.

\bibitem[Somvanshi et~al.(2024)Somvanshi, Javed, Islam, Pandit, and
  Das]{somvanshi2024survey}
Shriyank Somvanshi, Syed~Aaqib Javed, Md~Monzurul Islam, Diwas Pandit, and
  Subasish Das.
\newblock A survey on {Kolmogorov-Arnold} network.
\newblock \emph{arXiv preprint arXiv:2411.06078}, 2024.
\newblock URL \url{https://arxiv.org/abs/2411.06078}.

\bibitem[SS et~al.(2024)SS, AR, R, and KP]{ss2024chebyshev}
Sidharth SS, Keerthana AR, Gokul R, and Anas KP.
\newblock Chebyshev polynomial-based {Kolmogorov-Arnold} networks: An efficient
  architecture for nonlinear function approximation.
\newblock \emph{arXiv preprint arXiv:2405.07200}, 2024.

\bibitem[Ta(2024)]{ta2024bsrbfkan}
Hoang-Thang Ta.
\newblock {BSRBF-KAN}: A combination of {B}-splines and radial basis functions
  in {Kolmogorov-Arnold} networks.
\newblock \emph{arXiv preprint arXiv:2406.11173}, 2024.
\newblock URL \url{https://arxiv.org/abs/2406.11173}.

\bibitem[Toscano et~al.(2025)Toscano, Wang, and Karniadakis]{toscano2025kkans}
Juan~Diego Toscano, Li-Lian Wang, and George~Em Karniadakis.
\newblock {KKANs}: {Kurkova-Kolmogorov-Arnold} networks and their learning
  dynamics.
\newblock \emph{Neural Networks}, page 107831, 2025.

\bibitem[Wang et~al.(2020)Wang, Teng, and Perdikaris]{wang2020understanding}
Sifan Wang, Yujun Teng, and Paris Perdikaris.
\newblock Understanding and mitigating gradient pathologies in physics-informed
  neural networks. arxiv e-prints.
\newblock \emph{arXiv preprint arXiv:2001.04536}, 2020.

\bibitem[Wang et~al.(2021)Wang, Wang, and Perdikaris]{wang2021eigenvector}
Sifan Wang, Hanwen Wang, and Paris Perdikaris.
\newblock On the eigenvector bias of fourier feature networks: From regression
  to solving multi-scale {PDEs} with physics-informed neural networks.
\newblock \emph{Computer Methods in Applied Mechanics and Engineering},
  384:\penalty0 113938, 2021.

\bibitem[Wang et~al.(2023{\natexlab{a}})Wang, Sankaran, Wang, and
  Perdikaris]{wang2023expert}
Sifan Wang, Shyam Sankaran, Hanwen Wang, and Paris Perdikaris.
\newblock An expert's guide to training physics-informed neural networks.
\newblock \emph{arXiv preprint arXiv:2308.08468}, 2023{\natexlab{a}}.

\bibitem[Wang et~al.(2024{\natexlab{a}})Wang, Siegel, Liu, and
  Hou]{wang2024expressiveness}
Yixuan Wang, Jonathan~W Siegel, Ziming Liu, and Thomas~Y Hou.
\newblock On the expressiveness and spectral bias of {KANs}.
\newblock \emph{arXiv preprint arXiv:2410.01803}, 2024{\natexlab{a}}.

\bibitem[Wang et~al.(2024{\natexlab{b}})Wang, Sun, Bai, Anitescu, Eshaghi,
  Zhuang, Rabczuk, and Liu]{wang2024kolmogorov}
Yizheng Wang, Jia Sun, Jinshuai Bai, Cosmin Anitescu, Mohammad~Sadegh Eshaghi,
  Xiaoying Zhuang, Timon Rabczuk, and Yinghua Liu.
\newblock Kolmogorov {Arnold} informed neural network: A physics-informed deep
  learning framework for solving forward and inverse problems based on
  {Kolmogorov Arnold} networks.
\newblock \emph{arXiv preprint arXiv:2406.11045}, 2024{\natexlab{b}}.

\bibitem[Wang et~al.(2023{\natexlab{b}})Wang, Chen, and Chen]{wang2023solving}
Zhiwen Wang, Minxin Chen, and Jingrun Chen.
\newblock Solving multiscale elliptic problems by sparse radial basis function
  neural networks.
\newblock \emph{Journal of Computational Physics}, 492:\penalty0 112452,
  2023{\natexlab{b}}.

\bibitem[Wettschereck and Dietterich(1991)]{wettschereck1991improving}
Dietrich Wettschereck and Thomas Dietterich.
\newblock Improving the performance of radial basis function networks by
  learning center locations.
\newblock \emph{Advances in neural information processing systems}, 4, 1991.

\bibitem[Xu et~al.(1994)Xu, Krzy{\.z}ak, and Yuille]{xu1994radial}
Lei Xu, Adam Krzy{\.z}ak, and Alan Yuille.
\newblock On radial basis function nets and kernel regression: Statistical
  consistency, convergence rates, and receptive field size.
\newblock \emph{Neural Networks}, 7\penalty0 (4):\penalty0 609--628, 1994.

\bibitem[Yu et~al.(2024)Yu, Yu, and Wang]{yu2024kan}
Runpeng Yu, Weihao Yu, and Xinchao Wang.
\newblock {KAN} or {MLP}: A fairer comparison.
\newblock \emph{arXiv preprint arXiv:2407.16674}, 2024.

\bibitem[Zeng et~al.(2024)Zeng, Burghardt, and Gambaruto]{zeng2024rbf}
Chengxi Zeng, Tilo Burghardt, and Alberto~M Gambaruto.
\newblock {RBF-PINN}: Non-{Fourier} positional embedding in physics-informed
  neural networks.
\newblock \emph{arXiv preprint arXiv:2402.08367}, 2024.

\bibitem[Zheng et~al.(2025)Zheng, Zhang, Yue, Xu, Maennel, and
  Chen]{zheng2025free}
Liangwewi~Nathan Zheng, Wei~Emma Zhang, Lin Yue, Miao Xu, Olaf Maennel, and
  Weitong Chen.
\newblock Free-knots {Kolmogorov-Arnold} network: On the analysis of spline
  knots and advancing stability.
\newblock \emph{arXiv preprint arXiv:2501.09283}, 2025.

\end{thebibliography}

\pagebreak
\section*{Appendix}

\setcounter{section}{0}

\section{Full proof for \cref{thm:np-kan}}\label{sec:proof}

\begin{proof}
Let $f \in C([0,1]^d)$ and $\varepsilon > 0$ be given.
By Lemma \ref{lemma:KART}, $f$ 
admits the representation~\eqref{eq:KART} with 
continuous univariate functions $\{\Phi^{(q)}\}_{q=1}^{2d+1}$ and 
$\{\phi^{(pq)}\}_{1 \leq p \leq d, 1 \leq q \leq 2d+1}$.

For each $1 \leq q \leq 2d+1$, define inner sum $S_q: [0,1]^d \to \mathbb{R}$ by 
\begin{equation}
S_q(x) = \sum_{p=1}^{d} \phi^{(pq)}(x_p)  \text{ for all } x = (x_1, \cdots, x_d) \in [0,1]^d.
\end{equation}
Since $S_q$
is a continuous  on  the compact set $[0,1]^d$, 
its image $I_q = S_q([0,1]^d)$ is a compact interval in $\mathbb{R}$. 
By the Heine--Cantor theorem, $\Phi^{(q)}$ is uniformly continuous on $I_q$. 
Thus, there exists a $\delta_q > 0$ such that for all $y, \widehat{y} \in I_q$:
\begin{equation}
|y - \widehat{y}| < \delta_q \implies \left |\Phi^{(q)}(y) - \Phi^{(q)}(\widehat{y}) \right| < \frac{\varepsilon}{2(2d+1)}.
\label{eq:unif_cont}
\end{equation}
For each $1 \leq p \leq d$, by Lemma \ref{lemma:Leshno}, there exists $\widehat{\phi}^{(pq)} \in \mathcal{S}$ such that
\begin{equation}
\max_{x_p \in [0,1]} \left|\phi^{(pq)}(x_p) - \widehat{\phi}^{(pq)}(x_p)\right| < \frac{\delta_q}{d}.
\label{eq:inner_approx}
\end{equation}
Define $\widehat{S}_q: [0,1]^d \to \mathbb{R}$ by 
\begin{equation}
\widehat{S}_q(x_1,\ldots,x_d)
=
\sum_{p=1}^d \widehat{\phi}^{(pq)}(x_p) \text{ for all } x = (x_1, \cdots, x_d) \in [0,1]^d.
\end{equation}
Then
for all $x = (x_1, \cdots, x_d) \in [0,1]^d$,  
by the triangle inequality and \eqref{eq:inner_approx}, we have 
\begin{equation}
\left|S_q(x) - \widehat{S}_q(x)\right| \leq \sum_{p=1}^{d} \left| \phi^{(pq)}(x_p) - \widehat{\phi}^{(pq)}(x_p) \right| < \sum_{p=1}^d \frac{\delta_q}{d} = \delta_q.
\end{equation}
By \eqref{eq:unif_cont}, it follows that
\begin{equation}
\left| \Phi^{(q)}(S_q(x)) - \Phi^{(q)}(\widehat{S}_q(x))\right|  < \frac{\varepsilon}{2(2d+1)}.
\label{eq:inner_approx_val}
\end{equation}
Since $\widehat{S}_q$ is continuous on $[0,1]^d$,
$\widehat{I}_q=\widehat{S}_q([0,1]^d)$ is a compact interval.
Applying Lemma~\ref{lemma:Leshno} again, there exists $\widehat{\Phi}^{(q)} \in \mathcal{S}$ such that
\begin{equation}
\max_{y \in \widehat{I}_q} |\Phi^{(q)}(y) - \widehat{\Phi}^{(q)}(y)| < \frac{\varepsilon}{2(2d+1)} .
\label{eq:outer_approx_val}
\end{equation} 
Using triangle inequality and the estimates \eqref{eq:inner_approx_val} and 
\eqref{eq:outer_approx_val}, we arrive at 
\begin{equation}
\left| \Phi^{(q)}(S_q(x)) - \widehat{\Phi}^{(q)}(\widehat{S}_q(x)) \right| 
%\leq 
%\left| \Phi^{(q)}(S_q(x)) - \Phi^{(q)}(\widehat{S}_q(x)) \right| + 
%\left| \Phi^{(q)}(\widehat{S}_q(x)) - \widehat{\Phi}^{(q)}(\widehat{S}_q(x)) \right|
< \frac{\varepsilon}{2d+1}.
\end{equation}
Finally, summing over $q=1,\ldots,2d+1$ yields
\begin{equation}
\left| f(x) - g(x) \right| \leq \sum_{q=1}^{2d+1} \left| \Phi^{(q)}(S_q(x)) - \widehat{\Phi}^{(q)}(\widehat{S}_q(x)) \right| < \sum_{q=1}^{2d+1} \frac{\varepsilon}{2d+1} = \varepsilon,
\end{equation}
for all $x\in[0,1]^d$.
This concludes the proof.
\end{proof}

\pagebreak

\section{Solution visualization for experiments in Section~\ref{sec:piml}}

\begin{comment}
\begin{figure}[htp]
    \centering
    \includegraphics[width=0.5\textwidth]{regression_mnist_valacc.png}
    \caption{Validation loss during training on MNIST dataset.}
    \label{fig:mnist}
\end{figure}
\end{comment}

\begin{comment}
\begin{figure*}[htp]
    \includegraphics[width=\textwidth]{nonsmooth_RBFKANKnots.png}
    \caption{(Left) Analytical solution; (Middle) Prediction by Free-RBF-KAN; (Right) Error residual.}
    \label{fig:nonsmt-rbfkanknot}
\end{figure*}
\end{comment}

\begin{figure}[h!]
    \centering
    \begin{subfigure}[t]{0.32\textwidth}
        \caption{Exact}
        \includegraphics[width=\linewidth]{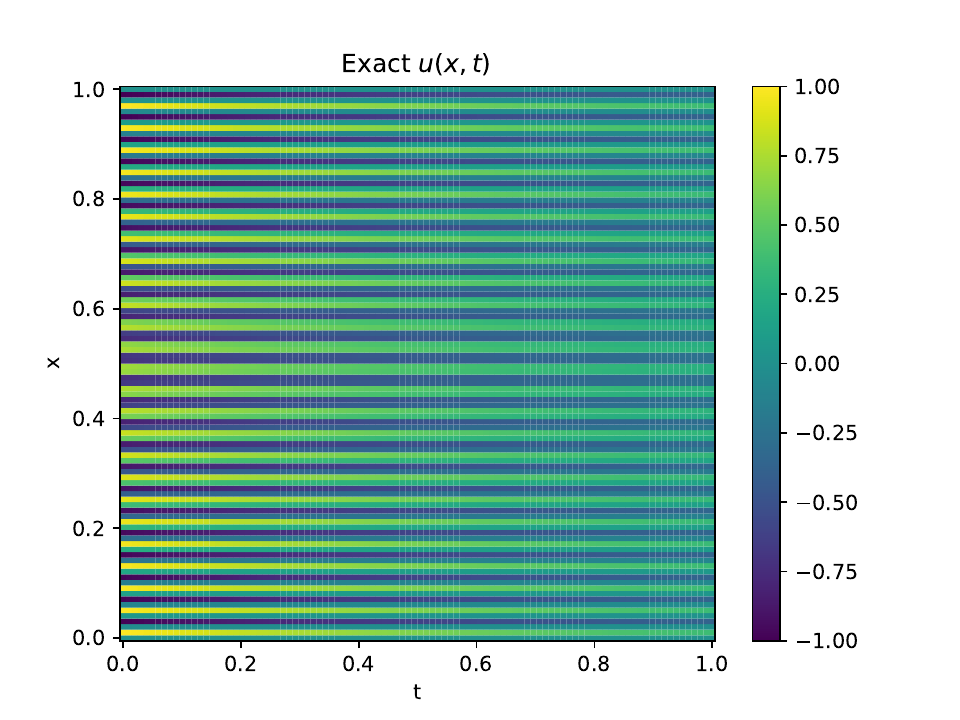}
        \label{fig:sub1}
    \end{subfigure}%
    \begin{subfigure}[t]{0.32\textwidth}
        \caption{MLP}
        \includegraphics[width=\linewidth]{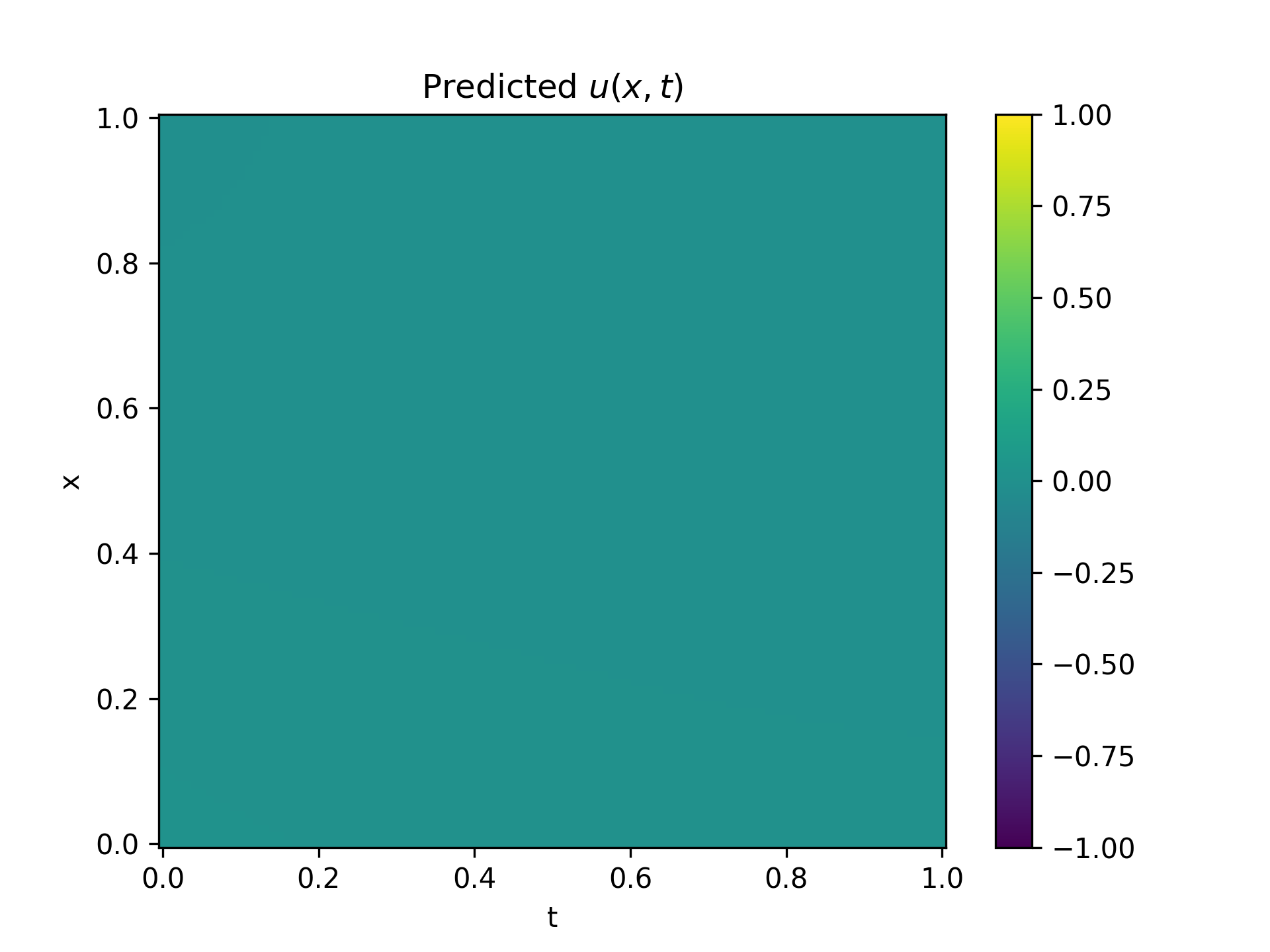}
        \label{fig:sub2}
    \end{subfigure}%
    \begin{subfigure}[t]{0.32\textwidth}
        \caption{Absolute Error}
        \includegraphics[width=\linewidth]{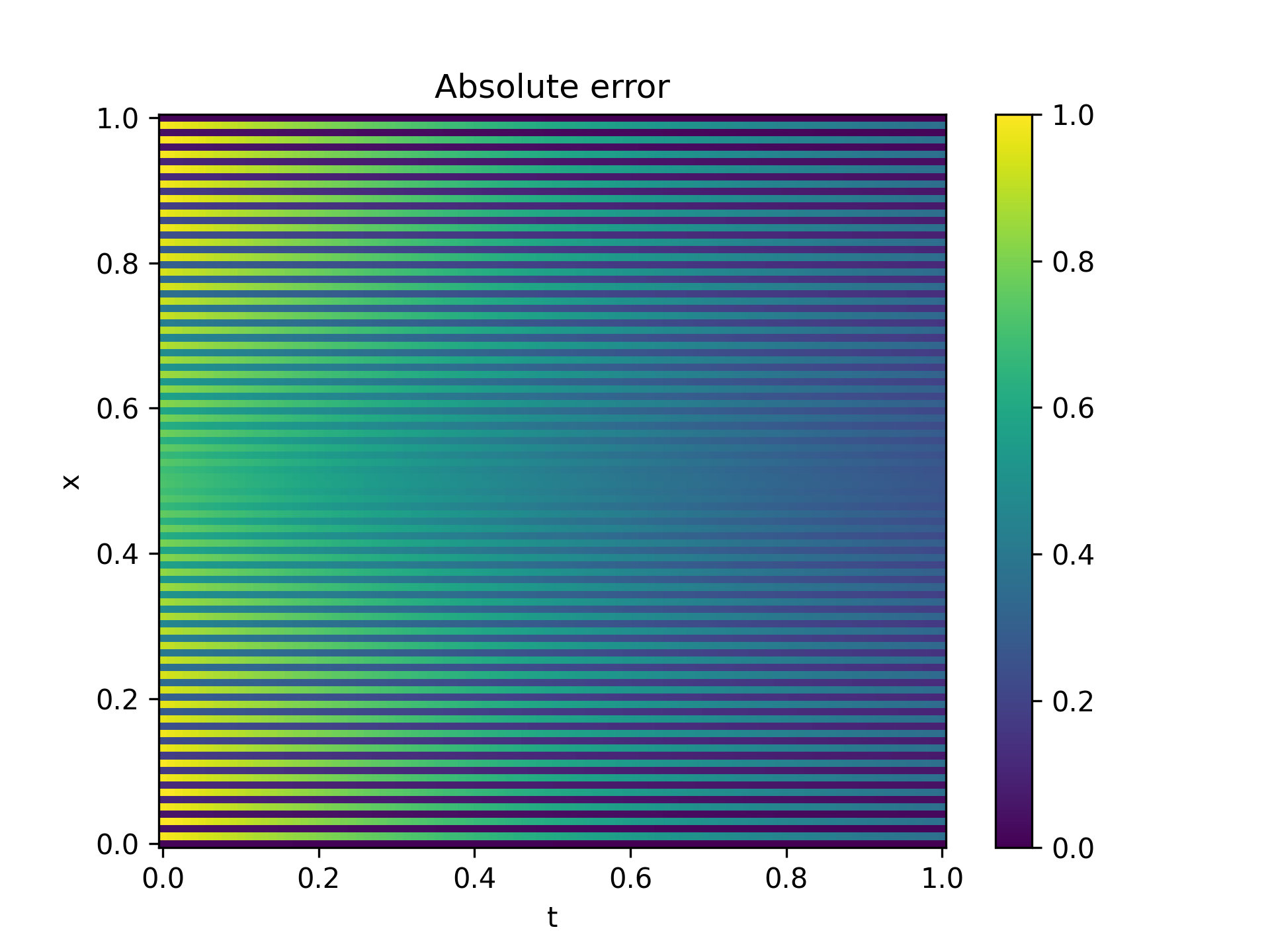}
        \label{fig:sub3}
    \end{subfigure}    
    \begin{subfigure}[t]{0.32\textwidth}
        \caption{Exact}
        \includegraphics[width=\linewidth]{heat_Exact_u.pdf}
        \label{fig:sub10}
    \end{subfigure}%
    \begin{subfigure}[t]{0.32\textwidth}
        \caption{KAN}
        \includegraphics[width=\linewidth]{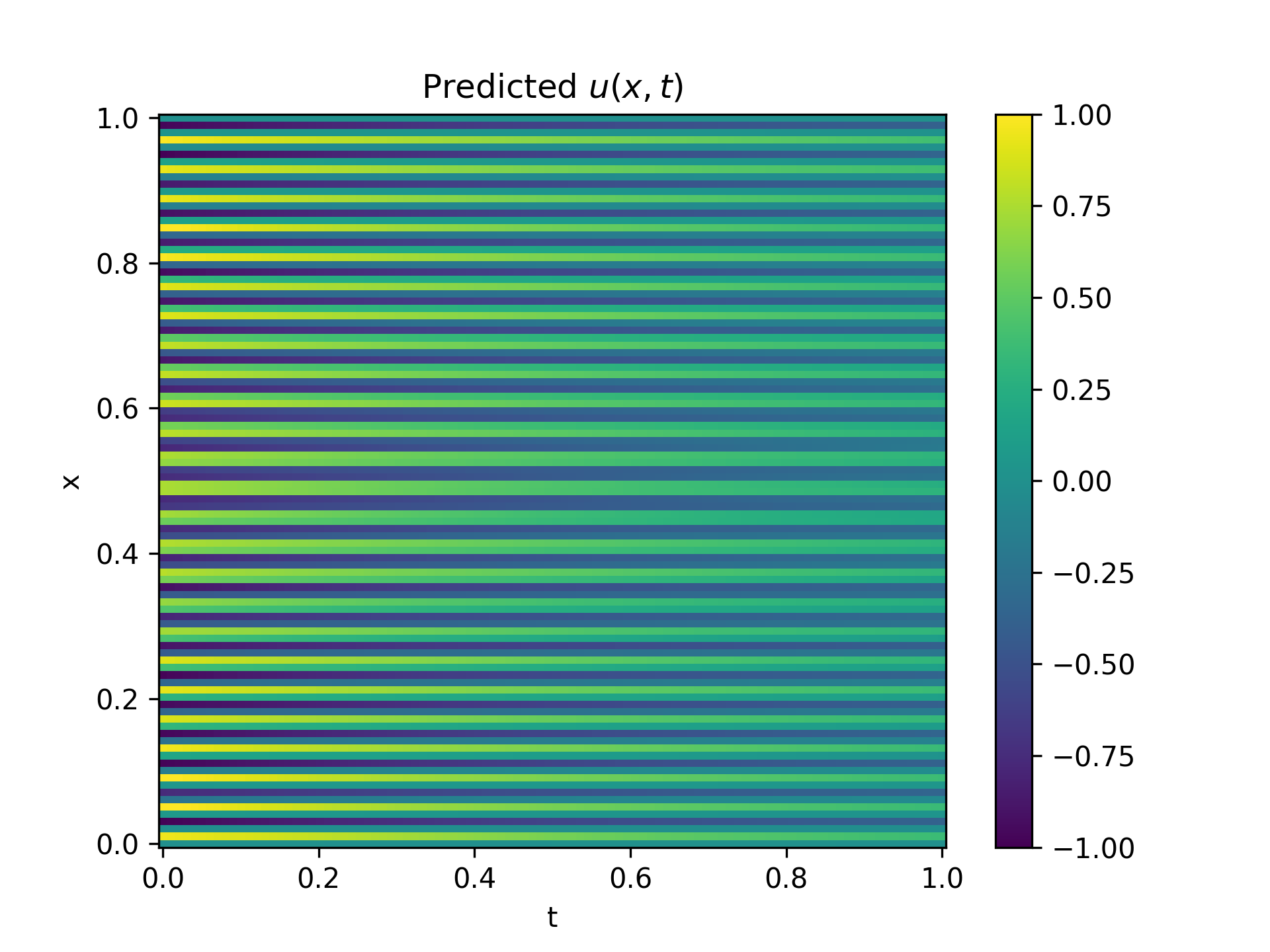}
        \label{fig:sub11}
    \end{subfigure}%
    \begin{subfigure}[t]{0.32\textwidth}
        \caption{Absolute Error}
        \includegraphics[width=\linewidth]{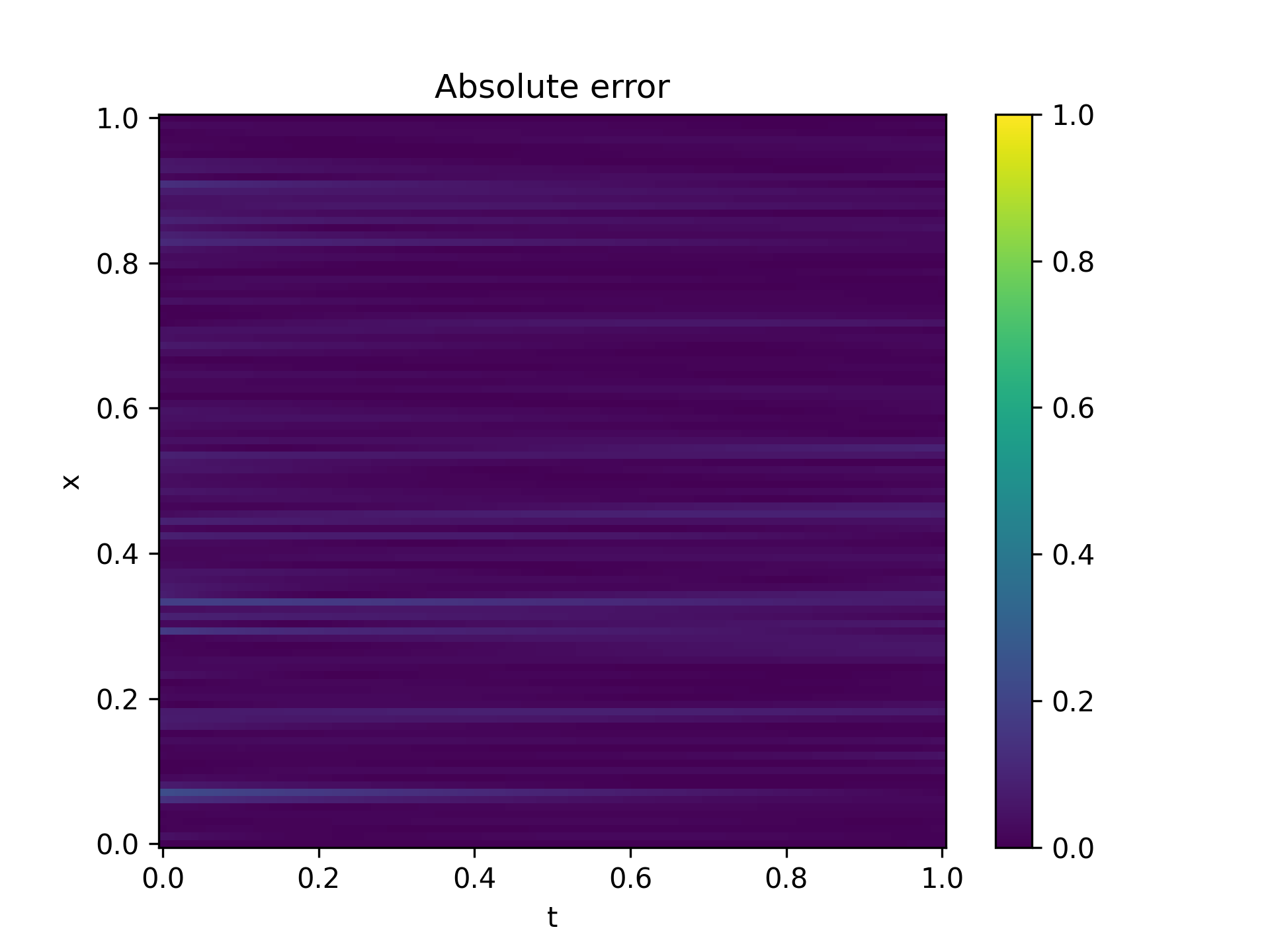}
        \label{fig:sub12}
    \end{subfigure} 
    \begin{subfigure}[t]{0.32\textwidth}
        \caption{Exact}
        \includegraphics[width=\linewidth]{heat_Exact_u.pdf}
        \label{fig:sub7}
    \end{subfigure}%
    \begin{subfigure}[t]{0.32\textwidth}
        \caption{RBF-KAN}
        \includegraphics[width=\linewidth]{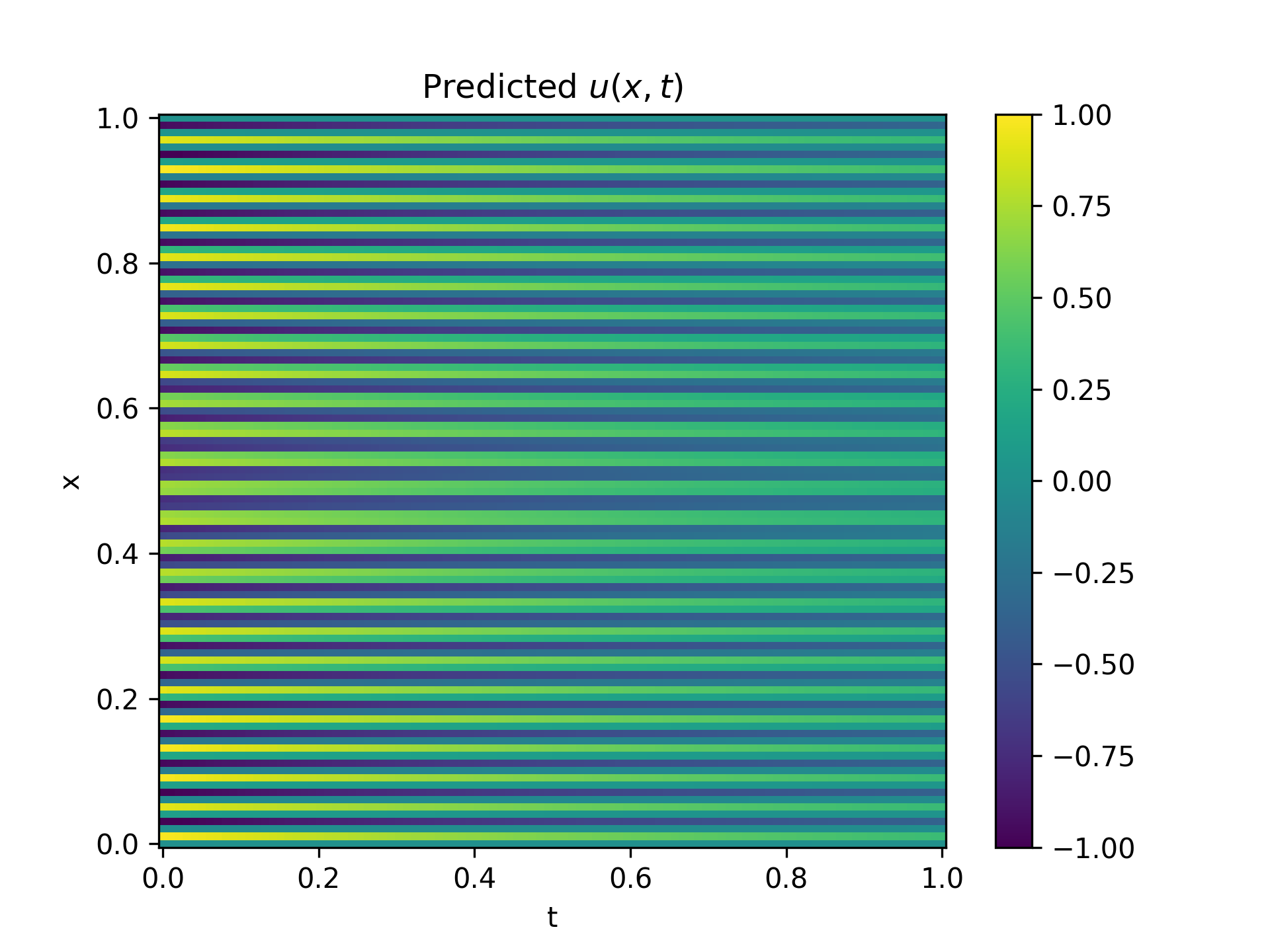}
        \label{fig:sub8}
    \end{subfigure}%
    \begin{subfigure}[t]{0.32\textwidth}
        \caption{Absolute Error}
        \includegraphics[width=\linewidth]{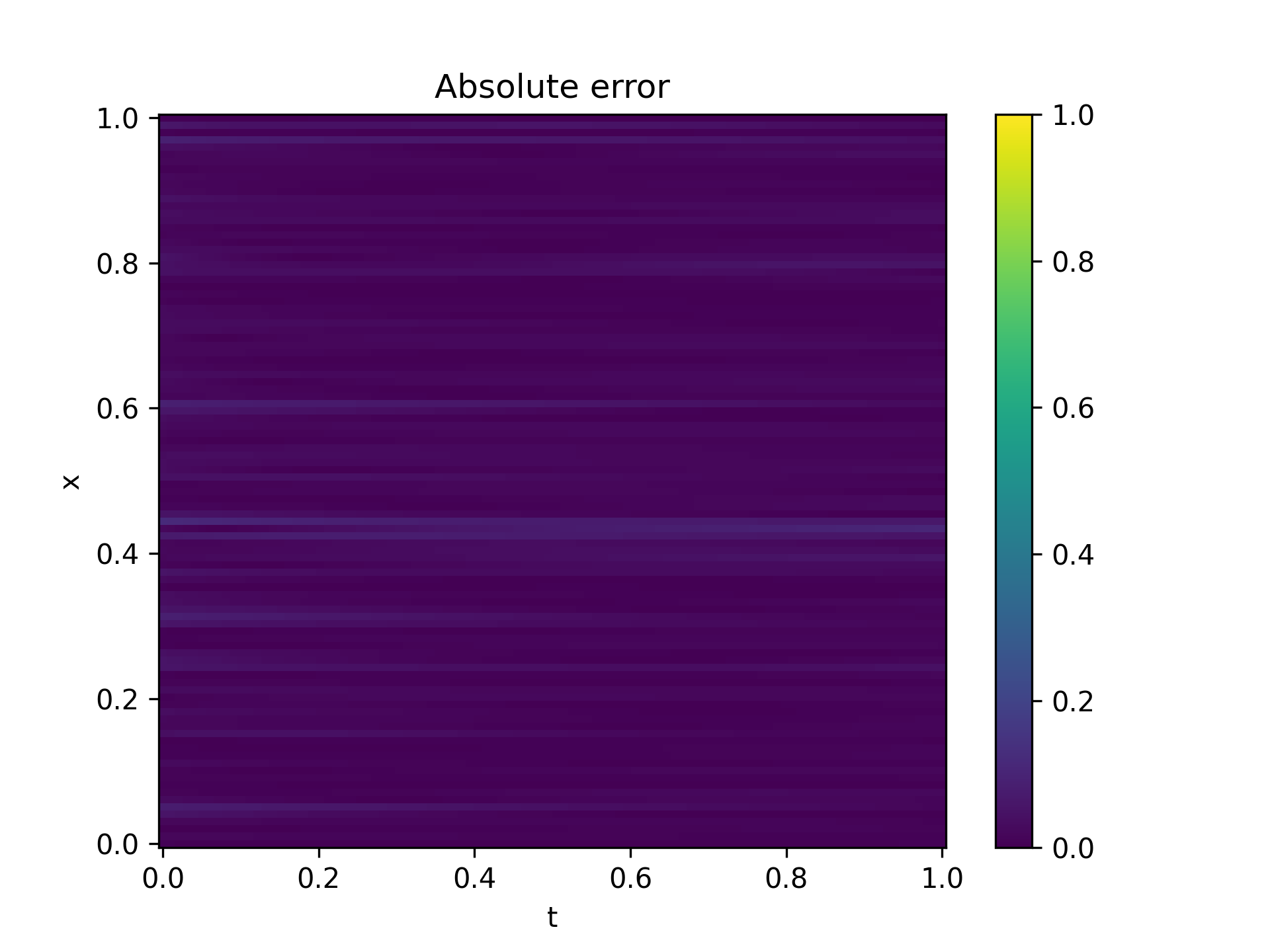}
        \label{fig:sub9}
    \end{subfigure} 
    \begin{subfigure}[t]{0.32\textwidth}
        \caption{Exact}
        \includegraphics[width=\linewidth]{heat_Exact_u.pdf}
        \label{fig:sub4}
    \end{subfigure}%
    \begin{subfigure}[t]{0.32\textwidth}
        \caption{Free-RBF-KAN}
        \includegraphics[width=\linewidth]{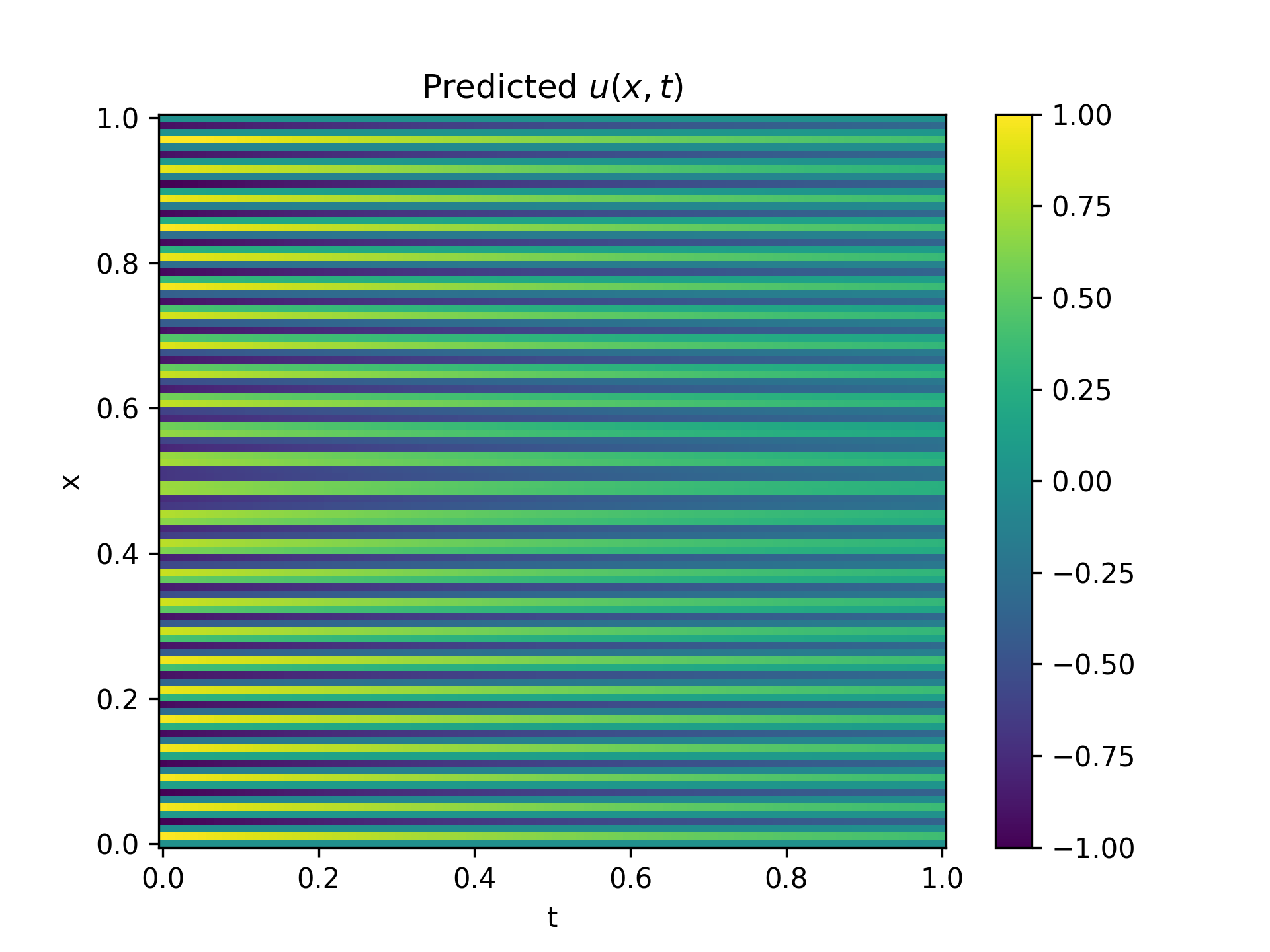}
        \label{fig:sub5}
    \end{subfigure}%
    \begin{subfigure}[t]{0.32\textwidth}
        \caption{Absolute Error}
        \includegraphics[width=\linewidth]{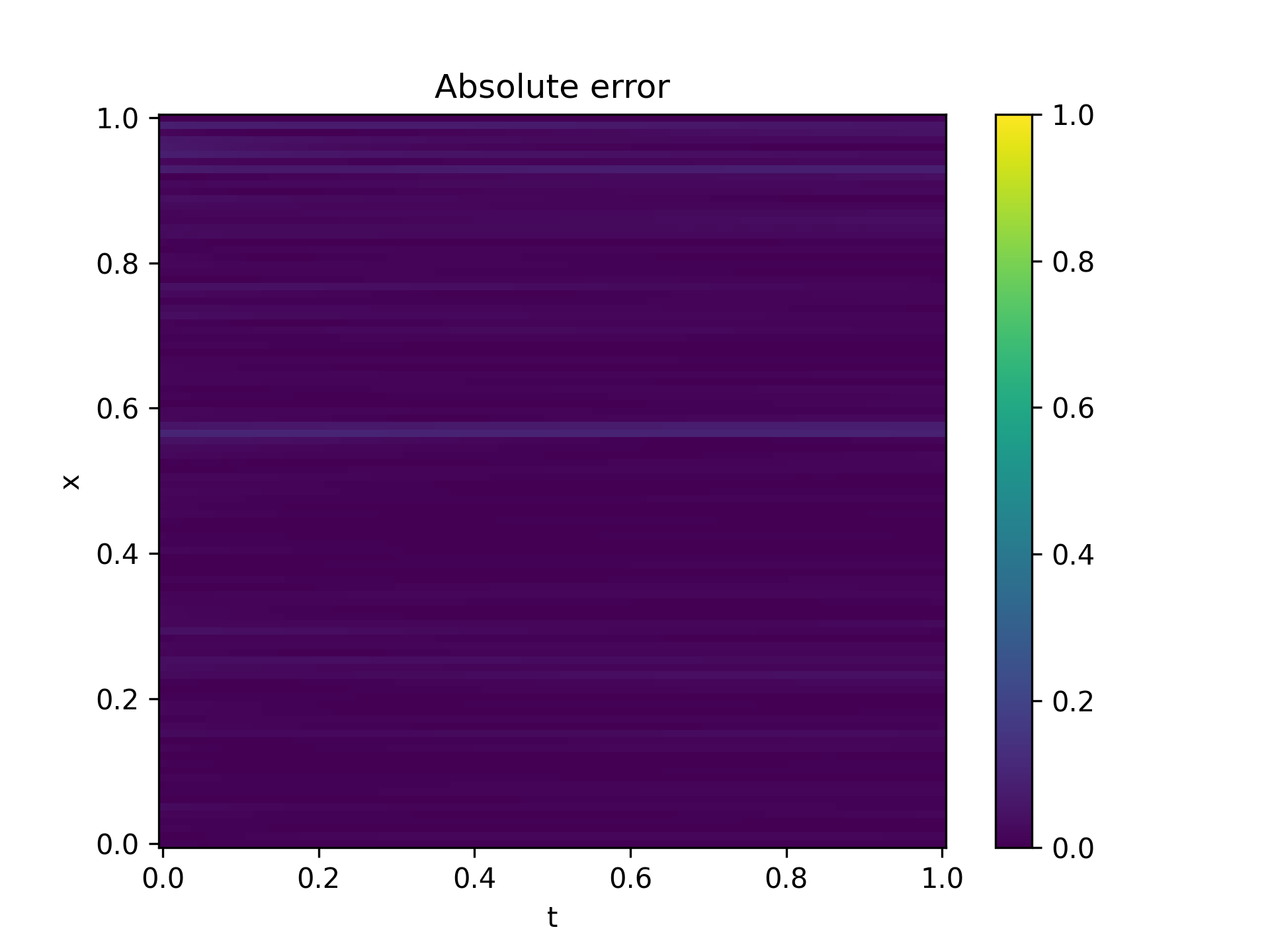}
        \label{fig:sub6}
    \end{subfigure}      
    \caption{Predicted solutions from MLP,  Free-RBF-KAN, RBF-KAN, and KAN on the heat conduction problem in Section~\ref{sec:heat}.}
    \label{fig:heat}
\end{figure}

\begin{figure}[h!]
    \centering
    \begin{subfigure}[t]{0.32\textwidth}
        \caption{Exact}
        \includegraphics[width=\linewidth]{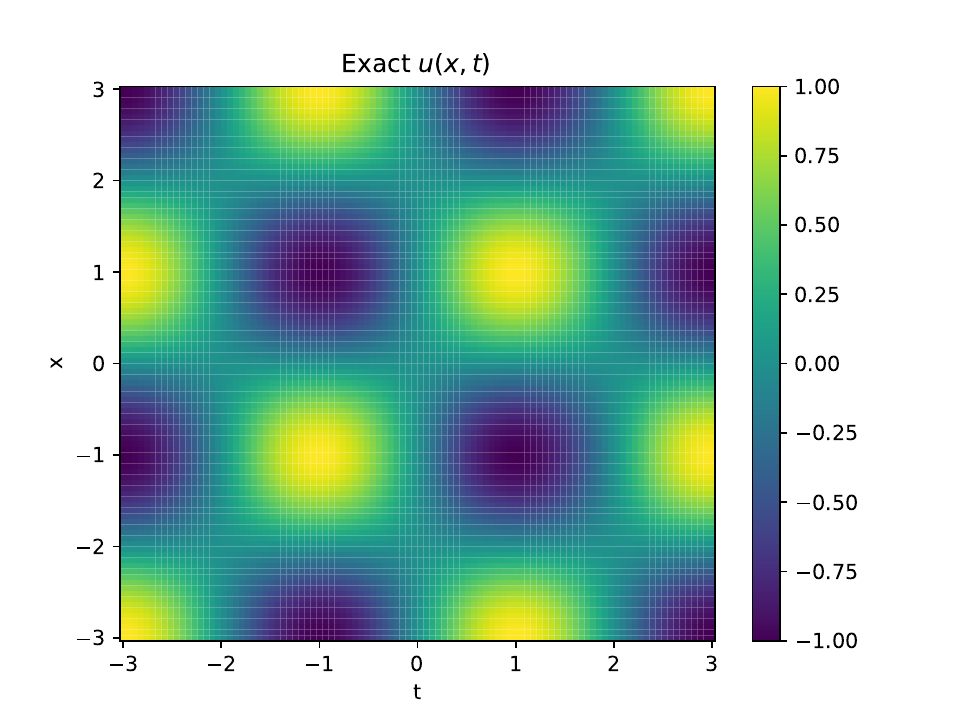}
        \label{fig:sub13}
    \end{subfigure}%
    \begin{subfigure}[t]{0.32\textwidth}
        \caption{MLP}
        \includegraphics[width=\linewidth]{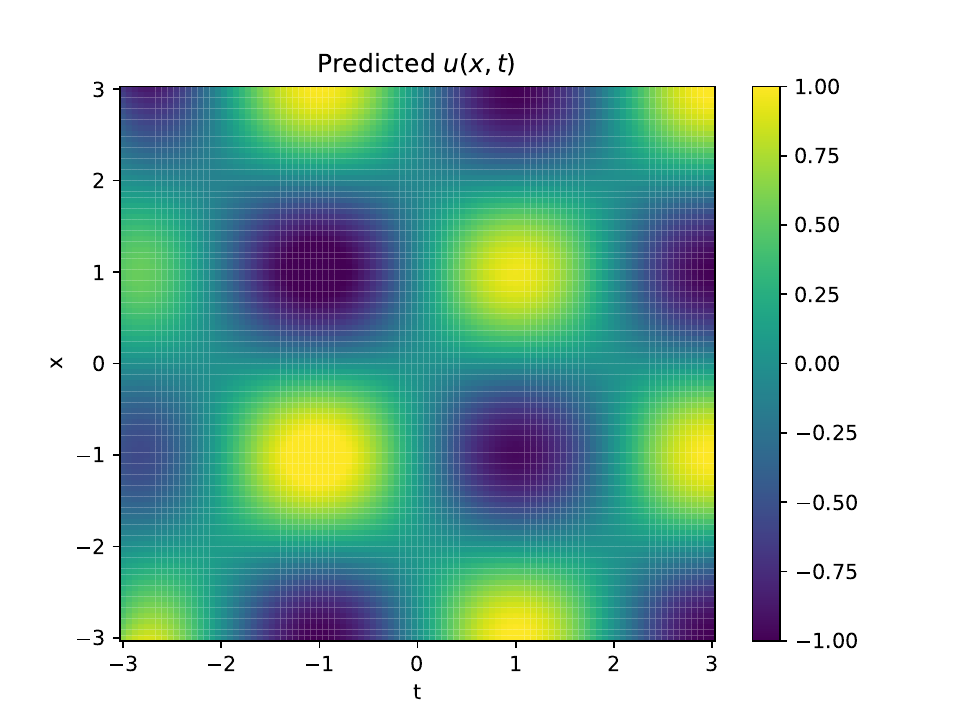}
        \label{fig:sub14}
    \end{subfigure}%
    \begin{subfigure}[t]{0.32\textwidth}
        \caption{Absolute Error}
        \includegraphics[width=\linewidth]{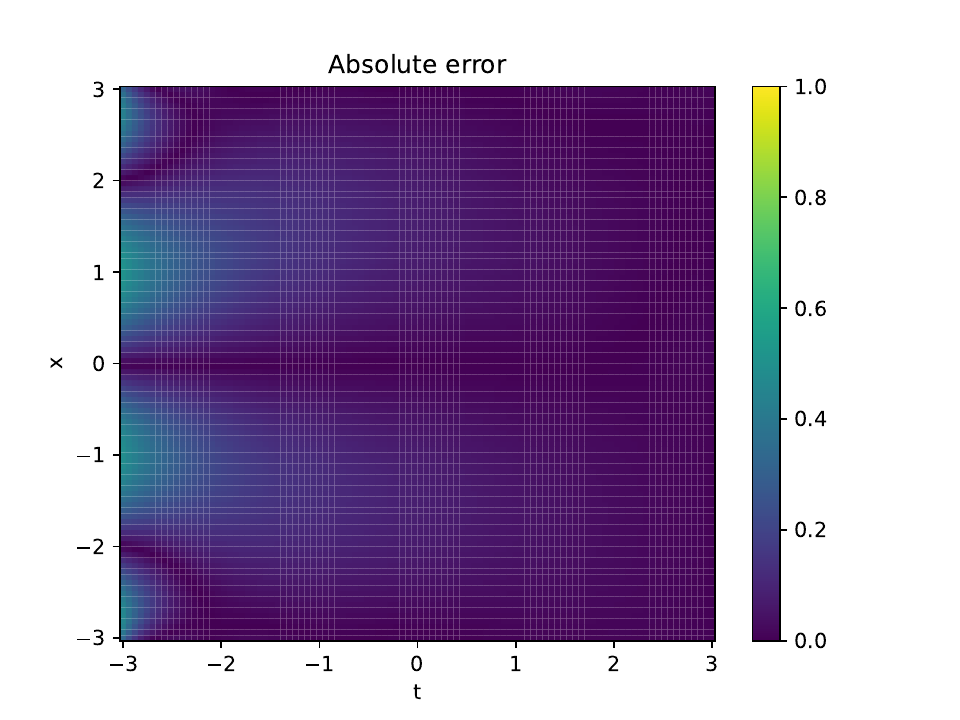}
        \label{fig:sub15}
    \end{subfigure}
    \begin{subfigure}[t]{0.32\textwidth}
        \caption{Exact}
        \includegraphics[width=\linewidth]{helmotz_Exact_u_KAN.pdf}
        \label{fig:sub22}
    \end{subfigure}%
    \begin{subfigure}[t]{0.32\textwidth}
        \caption{KAN}
        \includegraphics[width=\linewidth]{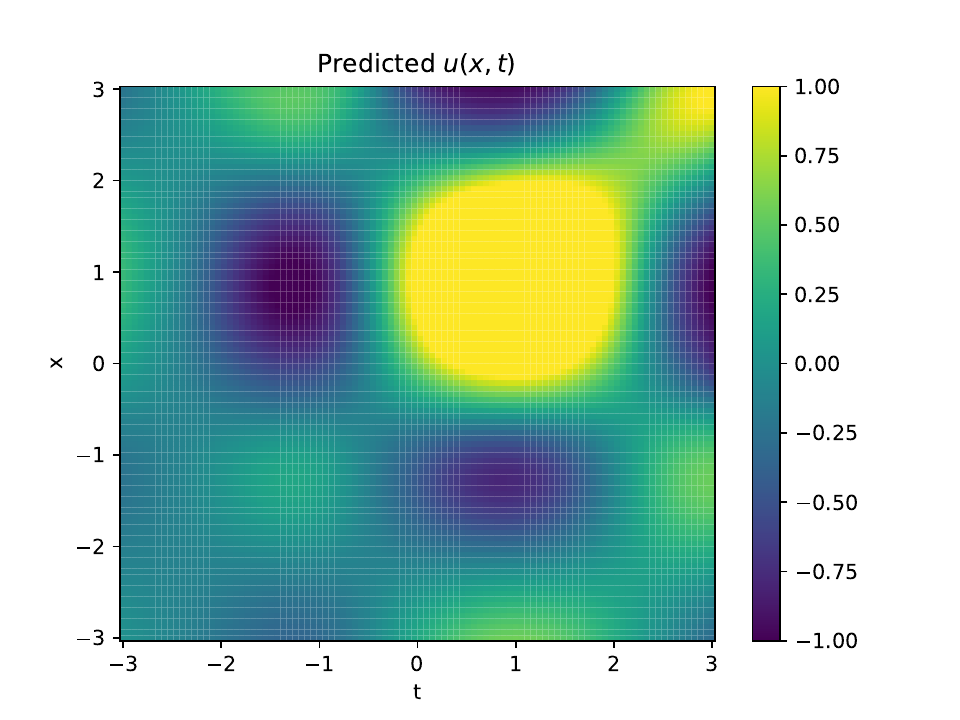}
        \label{fig:sub23}
    \end{subfigure}%
    \begin{subfigure}[t]{0.32\textwidth}
        \caption{Absolute Error}
        \includegraphics[width=\linewidth]{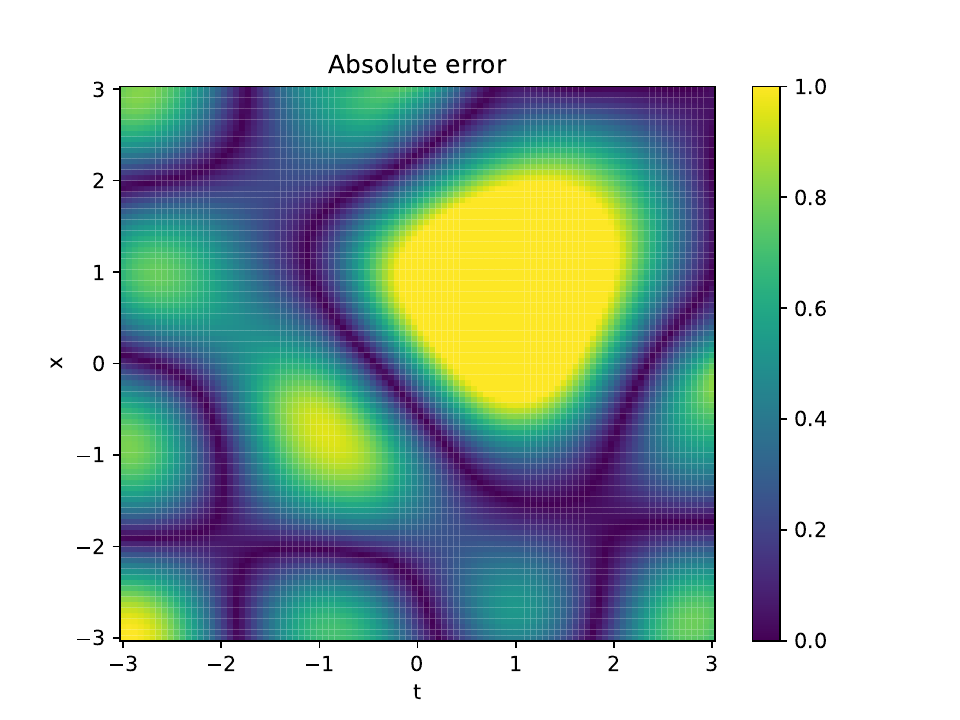}
        \label{fig:sub24}
    \end{subfigure}
    \centering
    \begin{subfigure}[t]{0.32\textwidth}
        \caption{Exact}
        \includegraphics[width=\linewidth]{helmotz_Exact_u_KAN.pdf}
        \label{fig:sub16}
    \end{subfigure}%
    \begin{subfigure}[t]{0.32\textwidth}
        \caption{RBF-KAN}
        \includegraphics[width=\linewidth]{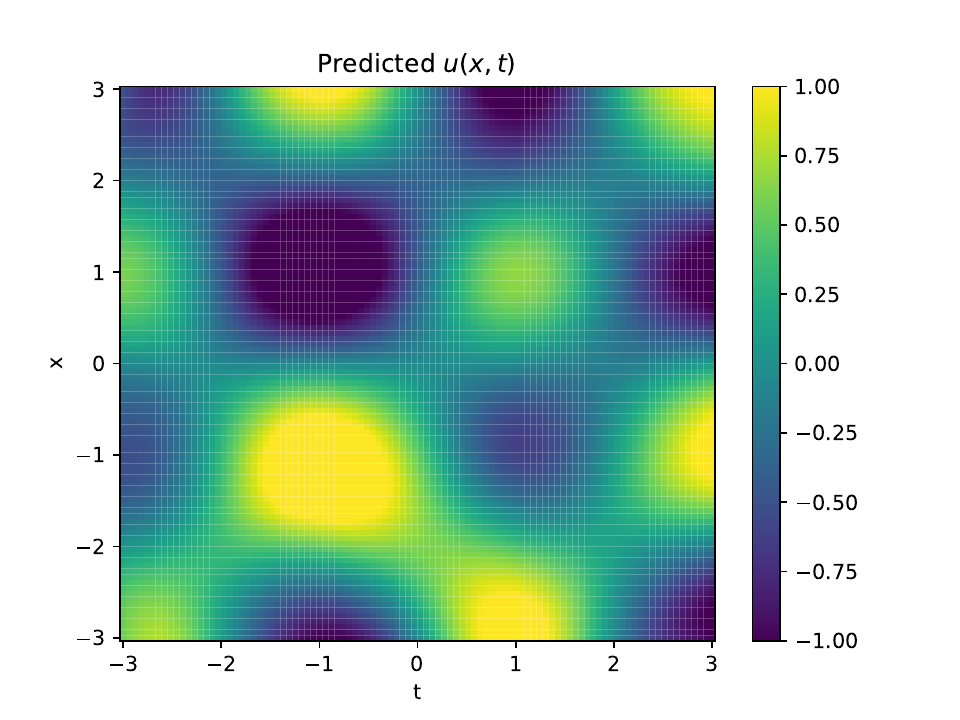}
        \label{fig:sub17}
    \end{subfigure}%
    \begin{subfigure}[t]{0.32\textwidth}
        \caption{Absolute Error}
        \includegraphics[width=\linewidth]{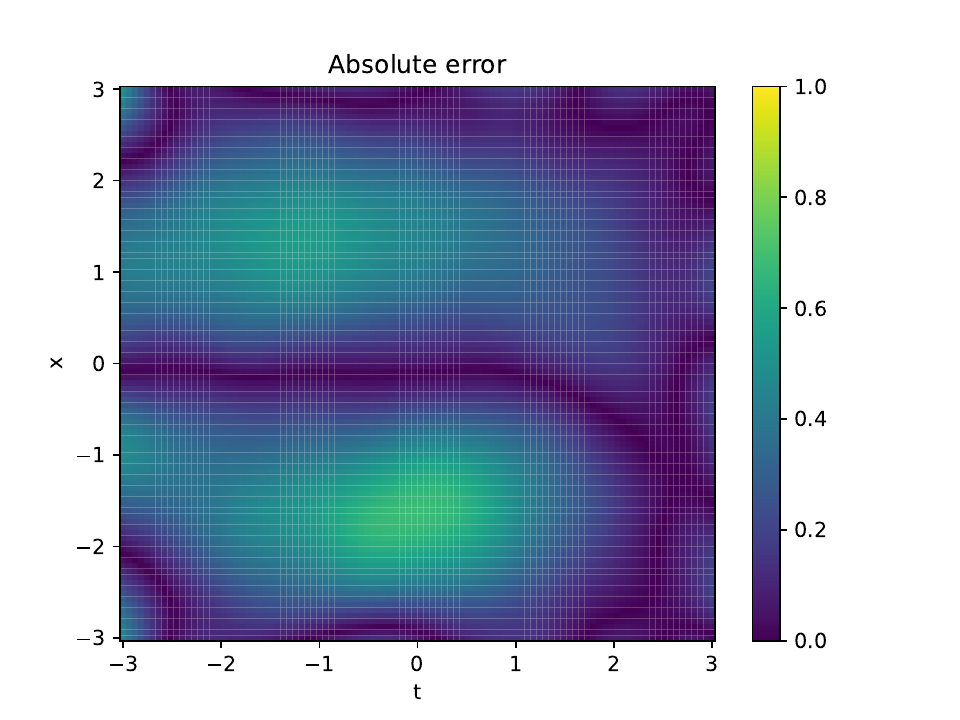}
        \label{fig:sub18}
    \end{subfigure}
    \centering
    \begin{subfigure}[t]{0.32\textwidth}
        \caption{Exact}
        \includegraphics[width=\linewidth]{helmotz_Exact_u_KAN.pdf}
        \label{fig:sub19}
    \end{subfigure}%
    \begin{subfigure}[t]{0.32\textwidth}
        \caption{Free-RBF-KAN}
        \includegraphics[width=\linewidth]{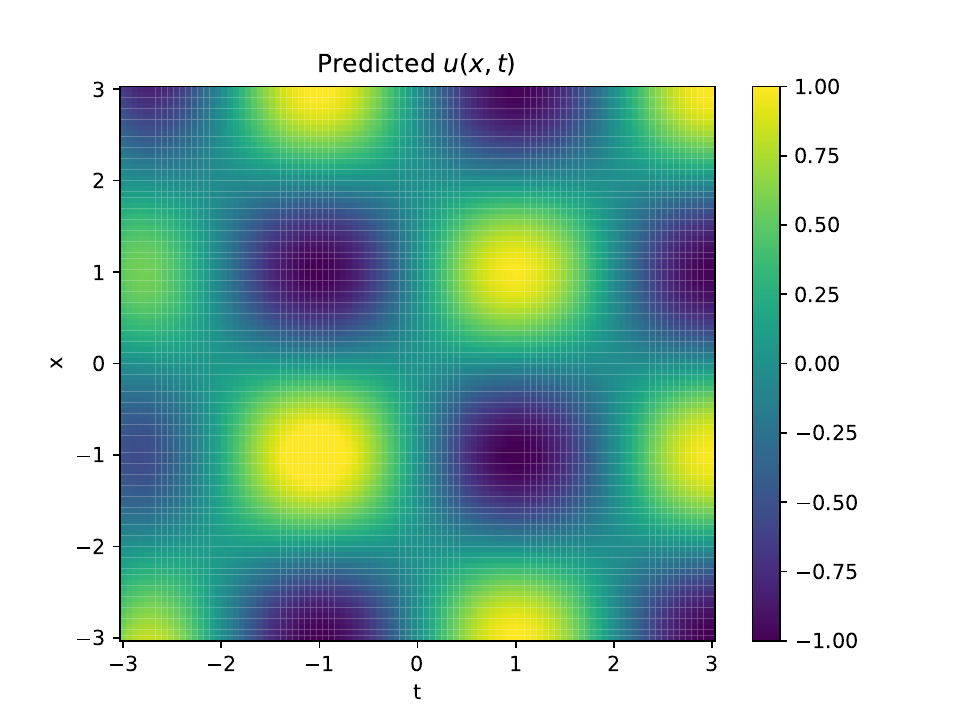}
        \label{fig:sub20}
    \end{subfigure}%
    \begin{subfigure}[t]{0.32\textwidth}
        \caption{Absolute Error}
        \includegraphics[width=\linewidth]{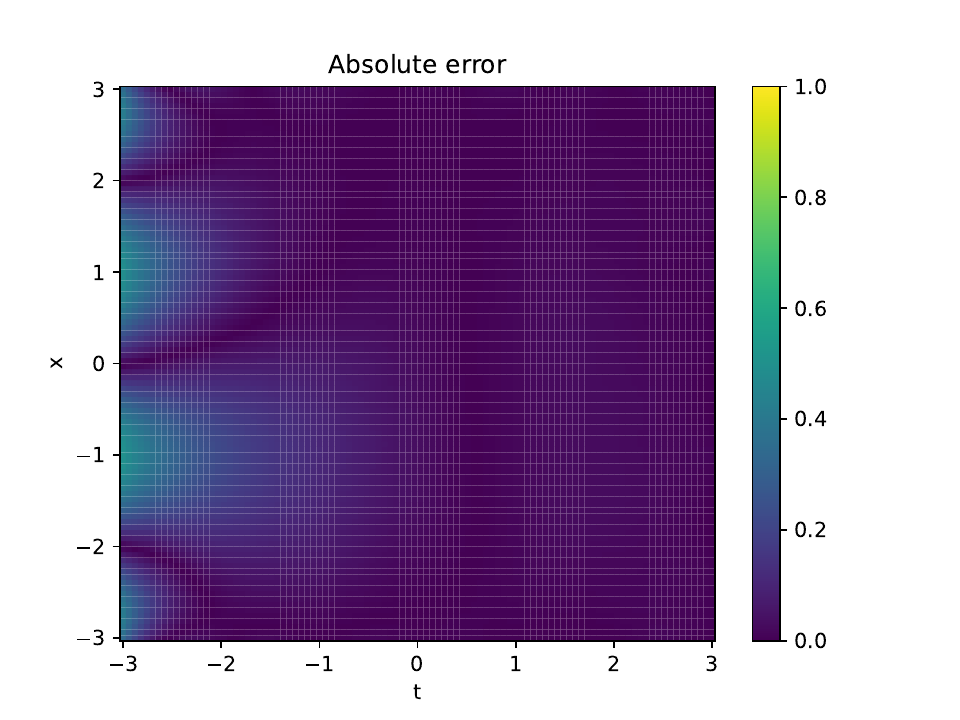}
        \label{fig:sub21}
    \end{subfigure}
    \centering
    \caption{
    The exact solutions (left column),
    the predicted solutions from MLP, RBF-KAN, Free-RBF-KAN and KAN (middle column),
    and the errors (right column)
    for the 2D Helmholtz equation in Section~\ref{sec:helm}.}
    \label{fig:helmholtz}
\end{figure}

\end{document}